\documentclass[11pt, logo, onecolumn, copyright, colorlinks=true, allcolors=blue]{nvidiatechreport}

\usepackage{microtype}
\usepackage{graphicx}
\usepackage{booktabs}
\usepackage{amsmath,amssymb}
\usepackage{mathtools}
\usepackage{hyperref}
\usepackage{url}
\usepackage{enumitem}
\usepackage{tabularx}
\usepackage{subcaption}
\usepackage{arydshln}
\usepackage{xcolor}
\usepackage{colortbl}
\usepackage{array}
\usepackage{makecell}

\usepackage{algorithm,algpseudocode}

\usepackage{graphicx} 
\usepackage{subcaption,ragged2e}
\usepackage[round]{natbib}

\usepackage[normalem]{ulem}
\usepackage[utf8]{inputenc} 
\usepackage[T1]{fontenc}    
\usepackage{url}
\usepackage{tikz}
\usepackage{enumitem}
\usetikzlibrary{positioning}
\usetikzlibrary{calc}
\usepackage{array}
\usepackage{booktabs}
\usepackage{wrapfig}
\usepackage{caption} 
\usepackage{listings}
\usepackage{adjustbox}
\usepackage{footnote}
\usepackage{multirow}
\usepackage{amsmath}
\usepackage{amsfonts}
\usepackage{breakcites}
\usepackage{blindtext}
\usepackage{svg}
\usepackage[most]{tcolorbox}
\usepackage{tabularx}
\usepackage{graphicx} 
\usepackage{xcolor}
\usepackage{colortbl}
\newcommand{\todo}[1]{\textcolor{red}{\textbf{[TODO}: #1]}}

\usepackage{amsthm}
\theoremstyle{plain}
\newtheorem{theorem}{Theorem}[section]

\newtheorem{proposition}[theorem]{Proposition}
\theoremstyle{definition}

\newcommand{\KL}{\mathrm{KL}}

\newcommand{\old}{\mathrm{old}}

\newcommand{\name}{\ifmmode\mathrm{PivotRL}\else PivotRL\fi}

\definecolor{b2}{RGB}{51,153,255}
\definecolor{myGreen}{RGB}{80,180,0}
\definecolor{myGold}{rgb}{0.75,0.6,0.12}
\definecolor{sweetgreen}{RGB}{198, 239, 161}

\title{PivotRL: High Accuracy Agentic Post-Training at Low Compute Cost}
\author{
Junkeun Yi$^{*, 1}$\quad
Damon Mosk-Aoyama$^{*, 1}$\quad
Baihe Huang$^{*, 2}$\quad
Ritu Gala$^{1}$\quad
Charles Wang$^{1}$\quad
Sugam Dipak Devare$^{1}$\quad
Khushi Bhardwaj$^{1}$\quad
Abhibha Gupta$^{1}$\quad
Oleksii Kuchaiev$^{1}$\quad
Jiantao Jiao$^{**,1,2}$\quad
Jian Zhang$^{**,1}$\quad
Venkat Srinivasan$^{**,\dagger,1}$
\vspace{6pt}

{\footnotesize\textnormal{$^{1}$NVIDIA\quad
$^{2}$UC Berkeley}}

{\footnotesize\textnormal{$^{*}$Equal Contribution\quad
$^{**}$Equal Advising\quad
$^{\dagger}$Corresponding Author: \texttt{venkats@nvidia.com}}}
}
\date{\todo{date}}

\begin{document}

\begin{abstract}
\large \textbf{Abstract.}
\normalsize
{Post-training for long-horizon agentic tasks has a tension between compute efficiency and generalization.}
While supervised fine-tuning (SFT) is compute efficient, it often suffers from out-of-domain (OOD) degradation. Conversely, end-to-end reinforcement learning (E2E RL) preserves OOD capabilities, but incurs high compute costs due to many turns of on-policy rollout.
We introduce \textbf{PivotRL}, a novel framework that operates on existing SFT trajectories to combine the compute efficiency of SFT with the OOD accuracy of E2E RL. PivotRL relies on two key mechanisms: first, it executes local, on-policy rollouts and filters for \emph{pivots}: informative intermediate turns where sampled actions exhibit high variance in outcomes; second, it utilizes rewards for functional-equivalent actions rather than demanding strict string matching with the SFT data demonstration.
{We theoretically show that these mechanisms incentivize strong  learning signals with high natural gradient norm, while maximally preserving policy probability ordering on actions unrelated to training tasks.}
In comparison to standard SFT on identical data, we demonstrate that
PivotRL achieves $+4.17\%$ higher in-domain accuracy on average across four agentic domains, and $+10.04\%$ higher OOD accuracy in non-agentic tasks.
{Notably, on agentic coding tasks, PivotRL achieves competitive accuracy with E2E RL with $4\times$ fewer rollout turns.}
{PivotRL is adopted by NVIDIA's Nemotron-3-Super-120B-A12B, acting as the workhorse in production-scale agentic post-training.}
\end{abstract}

\maketitle

\section{Introduction}

Long-horizon agentic tasks require many turns of Large Language Model (LLM) interaction with an environment. This includes tasks such as conversational tool use \citep{schick2023toolformer, qin2023toolllm}, agentic coding \citep{jimenez2024swebench}, terminal interaction \citep{xie2024osworld}, and web search \citep{yao2022webshop, wei2025browsecompsimplechallengingbenchmark}. These tasks have emerged as the frontier in AI systems capable of executing complex, real-world workflows. However, post-training for these long-horizon agentic capabilities introduces a fundamental tension between training efficiency and generalization. 
While supervised fine-tuning (SFT) offers a compute-efficient mechanism for acquiring new capabilities, it frequently struggles to generalize beyond the training distribution and catastrophically degrades out-of-domain (OOD) performance \citep{chu2025sftmemorizesrlgeneralizes, luo2023empiricalCF}. In contrast, end-to-end reinforcement learning (E2E RL) typically yields higher in-domain accuracy while robustly retaining OOD capabilities \citep{ouyang2022training, chen2025retainingdoingroleonpolicy}. Nevertheless, E2E RL incurs high compute overheads, as each parameter update necessitates repeated, many turns of on-policy rollout with environmental interactions. This dichotomy naturally raises the following question:
\begin{center}
    \emph{Can we combine the data efficiency of SFT with the generalization capabilities of E2E RL, achieving both in-domain accuracy and OOD retention without incurring full-trajectory rollouts?}
\end{center}
To address this challenge, a natural {attempt would be to repurpose SFT trajectory demonstrations for RL.} 
Specifically, one could sample on-policy rollouts conditioned on a randomly selected intermediate turn from an SFT trajectory, assigning a positive reward only if the generated action exactly matches the SFT data demonstration. While this mitigates the under-utilization of SFT data which only trains on a single demonstrated completion, our preliminary experiments reveal that it fails to improve OOD accuracy relative to standard SFT on the same data. We empirically trace this failure to two bottlenecks. First, randomly selected intermediate turns frequently provide negligible learning signals; under Group Relative Policy Optimization (GRPO)~\citep{shao2024deepseekmath}, sampled actions at such turns often uniformly succeed or fail, yielding a normalized advantage close to zero and thus producing no meaningful gradient update. Second, exact string matching with SFT data is excessively restrictive in generative action spaces, where numerous functionally equivalent actions validly diverge from the singular demonstrated completion. For instance, a tool call or search query may be perfectly appropriate despite not perfectly matching the demonstration. 

Motivated by these insights, we introduce \textbf{PivotRL}, a novel framework for long-horizon agentic RL to overcome both bottlenecks simultaneously. PivotRL executes local, on-policy rollouts and filters for \emph{pivots}: informative assistant turns where sampled actions exhibit mixed outcomes (i.e., both successes and failures). Consequently, training requires only brief, partial rollouts from these pivot states rather than exhaustive full trajectories. To effectively score these rollouts, we leverage domain-appropriate verifiers to assign rewards for functionally equivalent actions rather than strictly penalizing deviations from SFT data demonstrations \citep{shao2024deepseekmath, rohatgi2025tamingprocessverifiers}.

We substantiate our methodology with a lightweight theoretical analysis on our core design choices. 
{First, we prove that the Fisher norm of the natural gradient of the statewise reward objective scales with the reward standard deviation. Consequently, the GRPO update along the KL path scales directly with this variance, validating our strategy of filtering for mixed-outcome pivots to maximize the local in-domain learning signals.}
{Second, we show that functional reward-based RL shifts probability mass toward actions that are functionally equivalent to the expert demonstration, while preserving the conditional distributions on all other actions. This maintains the reference policy's relative ordering of task-unrelated actions, thereby mitigating out-of-domain (OOD) degradation.} 

Empirically, we first show that under identical training data---the same prompts and expert trajectories---PivotRL yields larger in-domain improvements than SFT while exhibiting substantially less OOD regression. 
{When training in the conversational tool use, agentic coding, terminal tool use and search domain,}
PivotRL achieves an average in-domain accuracy improvement of $+14.11$\footnote{Throughout this paper, performance scores and differences are reported in percentage points.}over the base model, which significantly improves the $+9.94$ points attained by SFT. Crucially, PivotRL does not introduce OOD regressions (+$0.21$ change), whereas SFT suffers a severe regression of $-9.48$ in non-agentic domains including math, science QA and competitive coding. Second, we provide a direct comparison with E2E RL on SWE-Bench, a rigorous standard for evaluating long-horizon agentic capabilities with well-established E2E RL baselines~\citep{jimenez2024swebench, wang2025openhandsopenplatformai}. On this benchmark, we demonstrate that PivotRL attains competitive accuracy with E2E RL, but only with a {$4\times$ fewer rollout turns.} 
{PivotRL is deployed at production scale for leading open LLMs. Along with SFT and E2E RL, it serves as the primary workhorse for agentic post-training for NVIDIA's Nemotron-3-Super LLM~\citep{nvidia2025nemotron3super}.}


\section{Preliminaries and Motivating Observations}
\label{sec:prelim}

\paragraph{Turn-level agentic training.}
Let $\tau$ be an interaction trajectory collected in agentic tasks. Decomposing $\tau$ at assistant decision boundaries gives $\tau=(s_{0},a_{0}^*,\ldots,s_{T},a_{T}^*)$, where $s_{t}$ is the full interaction history from the beginning of trajectory $\tau_i$ up to, but not including, the $t$-th assistant action, and $a_{t}^*$ is the demonstrated assistant completion at that state.
In turn-level agentic training, an \emph{action} is the full assistant completion at a model-call boundary, instead of an individual token. The standard negative log-likelihood loss for SFT is given by
\begin{align*}
    \mathcal{L}_{\mathrm{sft}}(\theta) = -\mathbb{E}_{\tau \sim \mathcal{D}_{\mathrm{sft}}, (s_t,a_t^*) \sim \tau} [\log \pi_\theta(a^*_t \mid s_t)].
\end{align*}
Let $\pi_0$ denote the reference policy, E2E RL samples rollouts from $\pi_{\old}$ and optimizes the GRPO~\citep{shao2024deepseekmath} objective given by:
\begin{align*}
    \mathcal{J}_{\mathrm{GRPO}}(\theta) = \mathbb{E}_{\substack{s \sim \mathcal{D}_{\pi_{\old}} \\ a_i \sim \pi_{\old}(\cdot \mid s)}} \left[ \frac{1}{G} \sum_{i=1}^G \min \left( w_i(\theta) \hat{A}_i, \operatorname{clip}(w_i(\theta), 1-\epsilon, 1+\epsilon) \hat{A}_i \right) - \beta \KL(\pi_\theta(\cdot \mid s) \parallel \pi_0(\cdot \mid s)) \right]
\end{align*}
where $w_i(\theta) = \frac{\pi_\theta(a_i \mid s)}{\pi_{\old}(a_i \mid s)}$ is the importance sampling weight, $\mathcal{D}_{\pi_{\old}}$ denotes the distribution of states $s_t$ from the on-policy rollout trajectories, and 
\begin{equation}
\hat{A}_i
=
\frac{
r_i-\frac{1}{G}\sum_{j=1}^G r_j
}{
\mathrm{std}(\{r_j\}_{j=1}^G) + \epsilon_{\mathrm{std}}
}.
\label{eq:group_adv}
\end{equation}
represents the advantage normalized across the group of $G$ sampled end-to-end rollouts.

\paragraph{Local RL from expert trajectories.}
To circumvent the computational overhead associated with full end-to-end trajectory generation, we consider a local RL paradigm. In this setting, instead of unrolling full interactions from the initial environment state until termination, we condition the policy on intermediate expert states $s_t$ derived directly from the SFT dataset and then conduct targeted, turn-level rollouts from these states. Given expert trajectory dataset $\mathcal{D}_{\mathrm{sft}}=\{\tau_i\}_{i=1}^N$, the simplest attempt to adapt it for local RL goes as follows: first sample a state $s_t$, then sample actions from $\pi_{\theta}(\cdot \mid s_t)$ and reward the action only when it exactly matches the demonstrated continuation:
\begin{equation}
r_{\mathrm{strict}}(s,a)=\mathbf{1}[a=a^*(s)].
\label{eq:strict_reward}
\end{equation}
This formulation represents the most direct translation of SFT demonstrations into a local RL framework: the interaction history is rigidly anchored to the expert trace, the subsequent action is sampled on-policy, and credit is sparsely assigned only for perfect replication of the demonstrated completion. However, empirical evaluations in $\tau^2$-Bench reveal that this naive local RL strategy yields merely marginal gains over standard behavior cloning on identical data, achieving an accuracy of $57.34$ compared to $58.44$ for same-data SFT. Consequently, this simplistic conversion of expert demonstrations into local RL episodes via exact-match reward functions proves inadequate for driving meaningful performance improvements.

\paragraph{Motivating observations.}
Through our preliminary experiments, we identify two bottlenecks in allocating rollout budgets and assigning local credit. \emph{First, turns with uniformly successful or failed actions are uninformative under group-normalized RL.} Indeed, if a batch of rewards $\{r_i\}_{i=1}^G$ consists entirely of zeros or ones, then the normalized advantage in Eq.~\eqref{eq:group_adv} evaluates to zero. Empirically on $\tau^2$-bench and SWE-Bench, $71\%$ of randomly sampled turns yield a $0$ learning signal, meaning they are uniformly solved or uniformly failed and contribute nothing to the gradient. \emph{Second, exact-match local credit is too strict.} In generative action spaces, many tool calls, shell commands, or search steps are locally acceptable without matching the single demonstrated string exactly. Comparing exact matching to a more permissive verifier-based reward $r_{\mathrm{func}}$ (introduced in Section~\ref{sec:pivotrl_method}), we define the miss rate as $\Pr\!\left[
r_{\mathrm{strict}}(s,a)=0
\;\middle|\;
r_{\mathrm{func}}(s,a)=1
\right].$
A high miss rate means that exact matching erroneously discards rollouts that are functionally correct at the local decision point. These two bottlenecks map directly to the two ingredients of PivotRL: we first filter for \emph{pivots} (informative turns that continue to produce mixed outcomes), and then replace exact-match local credit with a verifier that rewards locally acceptable actions.

\section{PivotRL}
\label{sec:pivotrl}

PivotRL modifies the naive local-RL baseline from Section~\ref{sec:prelim} in exactly two ways: it filters extracted turns so that online rollout budget is spent on informative states, and it replaces exact-match local credit with verifier-based reward. We now introduce the full training pipeline and then give two theoretical results that help explain these choices. Section~\ref{sec:pivotrl_method} presents the method. Section~\ref{sec:pivotrl_analysis} then studies the two components separately: Proposition~\ref{prop:mixed} and Theorem~\ref{thm:pivot} explain why turns with very different responses provide stronger local learning signal under group-normalized RL, while Theorem~\ref{thm:kl} shows that the verifier-based reward shifts probability mass toward acceptable actions while remaining conservative relative to the reference policy.

\subsection{Method}
\label{sec:pivotrl_method}

In turn-level training, we extract assistant turns from each trajectory $\tau=(s_0,a_0^*,\ldots,s_T,a_T^*)$ into a \emph{pivot candidate} dataset:
\begin{equation}
\mathcal{D}_{\mathrm{cand}}=\{(s_t,a_t^*)\}_{\tau \in \mathcal{D}_{\mathrm{SFT}}, ~t = 0,\dots,T}.
\label{eq:cand}
\end{equation}
PivotRL performs three steps: (i) profile turns offline and retain only those likely to remain informative, (ii) sample local on-policy rollouts at the retained turns, and (iii) optimize a verifier-based GRPO-style objective. We summarize the full procedure in Algorithm~\ref{alg:pivotrl}. The first step addresses the uninformative-turn bottleneck from Section~\ref{sec:prelim}; the second addresses the overly strict local-credit bottleneck.

\paragraph{Offline turn selection.}
We estimate the informativeness of each extracted turn under a frozen reference policy $\pi_0$, typically the policy used to initialize PivotRL. For a turn state $s$, we sample $K$ local rollouts from $\pi_0(\cdot\mid s)$, score them with the verifier, and compute
\begin{equation}
\hat{\mu}(s)
=
\frac{1}{K}\sum_{k=1}^K r_{\mathrm{func}}(s,a^{(k)}),
\qquad
\hat{\sigma}^2(s)
=
\frac{1}{K}\sum_{k=1}^K \bigl(r_{\mathrm{func}}(s,a^{(k)})-\hat{\mu}(s)\bigr)^2.
\label{eq:profile_stats}
\end{equation}
We then keep only turns with nonzero empirical reward variance and low reward mean, 
\begin{align}
\mathcal{D}_{\mathrm{adv}}
=
\left\{
(s,a^*)\in \mathcal{D}_{\mathrm{cand}}
\;:\;
\hat{\sigma}^2(s)>0, ~~
\hat{\mu}(s)<\lambda_{\mathrm{diff}}
\right\}.
\label{eq:dadv}
\end{align}
The first filter removes turns that are already uniformly solved or uniformly failed under the reference policy; the second concentrates training on mixed-outcome turns that are still difficult. This filtered subset is thus called \emph{pivot} used for PivotRL training.
We write $\mathcal{D}_{\mathrm{pivot}}$ for the retained training set, and unless otherwise stated use $\mathcal{D}_{\mathrm{pivot}}=\mathcal{D}_{\mathrm{adv}}$.

\paragraph{Verifier-based local reward.}
For a retained state $s$ (present in $\mathcal{D}_{adv}$), let $\mathcal{M}(s)\subseteq \mathcal{A}(s)$ denote the set of locally acceptable actions under a domain-specific verifier. PivotRL assigns reward
\begin{equation}
r_{\mathrm{func}}(s,a)=\mathbf{1}[a\in \mathcal{M}(s)].
\label{eq:func_reward}
\end{equation}
Relative to the strict local reward in Eq.~\eqref{eq:strict_reward}, this verifier credits any action that is acceptable at the current turn, not only the single demonstrated completion. Depending on the domain, the verifier may be a normalized string/schema check, a task-specific equivalence rule, or a lightweight LLM judge.

Given a retained turn set $\mathcal{D}_{\mathrm{pivot}}$ and rollout group size $G$, PivotRL samples $\{a_i\}_{i=1}^G \sim \pi_{\theta_{\mathrm{old}}}(\cdot\mid s)$ at each selected state and optimizes
\begin{equation}
\mathcal{J}_{\mathrm{PivotRL}}(\theta)
=
\mathbb{E}_{\substack{
s\sim \mathcal{D}_{\mathrm{pivot}}\\
\{a_i\}_{i=1}^G \sim \pi_{\theta_{\mathrm{old}}}(\cdot\mid s)
}}
\left[
\frac{1}{G}
\sum_{i=1}^G
\min\!\Bigl(
w_i(\theta)\hat{A}_i,\,
\mathrm{clip}(w_i(\theta),1-\epsilon,1+\epsilon)\hat{A}_i
\Bigr)
- D_{\text{KL}}
\right],
\label{eq:pivotrl_obj}
\end{equation}
where $D_{\text{KL}} = \beta \KL(\pi_\theta(\cdot \mid s) \parallel \pi_0(\cdot \mid s))$ and
\begin{equation}
w_i(\theta)
=
\frac{\pi_\theta(a_i\mid s)}{\pi_{\theta_{\mathrm{old}}}(a_i\mid s)},
\label{eq:importance_weight}
\end{equation}
and $\hat{A}_i$ is the group-normalized advantage from Eq.~\eqref{eq:group_adv}, computed using the local verifier rewards $\{r_{\mathrm{func}}(s,a_i)\}_{i=1}^G$. Relative to end-to-end RL, the only online interaction during training is the short rollout needed to score each sampled turn-level action.

\begin{algorithm}[t]
\caption{PivotRL}
\label{alg:pivotrl}
\begin{algorithmic}[1]
\State \textbf{Require:} expert trajectories $\mathcal{D}_{\mathrm{sft}}$, reference policy $\pi_0$, initial policy $\pi_\theta$, verifier $V$, rollout group size $G$, difficulty threshold $\lambda_{\mathrm{diff}}$
\State Extract turns $\mathcal{D}_{\mathrm{cand}}=\{(s_t,a_t^*)\}$ from $\mathcal{D}_{\mathrm{sft}}$
\State Profile each $s\in\mathcal{D}_{\mathrm{cand}}$ with Eq.~\eqref{eq:profile_stats}
\State Set $\mathcal{D}_{\mathrm{pivot}}\leftarrow \mathcal{D}_{\mathrm{adv}}$
\For{$k=1,\dots,K_{\mathrm{train}}$}
    \State Sample minibatch $\{s_b\}_{b=1}^B \sim \mathcal{D}_{\mathrm{pivot}}$
    \For{each $s_b$}
        \State Sample $\{a_{b,i}\}_{i=1}^G \sim \pi_{\theta_{\mathrm{old}}}(\cdot\mid s_b)$
        \State Execute the short rollout needed to score each $a_{b,i}$
        \State $r_{b,i}\leftarrow r_{\mathrm{func}}(s_b,a_{b,i})$ for $i=1,\dots,G$
        \State Compute $\hat{A}_{b,i}\leftarrow \dfrac{r_{b,i}-\frac{1}{G}\sum_{j=1}^G r_{b,j}}{\mathrm{std}(\{r_{b,j}\}_{j=1}^G) + \epsilon_{\mathrm{std}}}$
    \EndFor
    \State Update $\theta$ using the objective in Eq.~\eqref{eq:pivotrl_obj}
\EndFor
\end{algorithmic}
\end{algorithm}

\subsection{Theoretical analysis for PivotRL}
\label{sec:pivotrl_analysis}

We next analyze the two design choices from Section~\ref{sec:pivotrl_method}. We first formalize why mixed-outcome turns are the right states for group-normalized local RL. We then show that verifier-based reward yields a conservative KL-regularized update: it increases total mass on locally acceptable actions while preserving the reference ordering within and outside that acceptable set. Throughout this subsection, we assume finite action spaces, $\mathcal{M}(s)\neq\emptyset$, and $\pi_0(a\mid s)>0$ for all $s$ and $a$.

\begin{proposition}[Only mixed-outcome turns produce nonzero group-normalized updates]
\label{prop:mixed}
Let $s$ be a fixed state and let $\{r_i\}_{i=1}^G$ be the binary rewards of a rollout group at $s$. If all rewards are identical, then the normalized advantages in Eq.~\eqref{eq:group_adv} are zero for every $i$. Equivalently, only rollout groups with positive reward variance can contribute a nonzero group-normalized update.
\end{proposition}

Proposition~\ref{prop:mixed} is the direct reason to filter turns before RL: if a turn is uniformly easy or uniformly impossible under local sampling, then spending rollout budget on that turn does not change the policy under group-normalized training.

For any distribution $\pi$ over $\mathcal{A}(s)$, let
\[
T_\pi
=
\left\{
v:\mathcal{A}(s)\to\mathbb{R}
\;:\;
\mathbb{E}_{a\sim\pi}[v(a)]=0
\right\},
\qquad
\langle u,v\rangle_{F,\pi}
=
\mathbb{E}_{a\sim\pi}[u(a)v(a)],
\]
and let $\nabla^{\mathrm{nat}}J_s(\pi)\in T_\pi$ denote the natural gradient of $J_s$ under this Fisher geometry.

\begin{theorem}[Reward variance determines the local GRPO signal]
\label{thm:pivot}
Consider the statewise expected reward objective
\[
J_s(\pi)=\mathbb{E}_{a\sim\pi}[r(s,a)],
\]
and the KL path
\[
\pi_{s,\beta}(a)
=
\frac{
\pi_0(a\mid s)e^{r(s,a)/\beta}
}{
\sum_{b\in \mathcal{A}(s)} \pi_0(b\mid s)e^{r(s,b)/\beta}
}.
\]
Define the population GRPO score
\begin{equation}
\gamma_{s,\beta}
:=
-
\mathbb{E}_{a\sim \pi_{s,\beta}}
\left[
\frac{
r(s,a)-\mathbb{E}_{a'\sim\pi_{s,\beta}}[r(s,a')]
}{
\sqrt{\mathrm{Var}_{a'\sim\pi_{s,\beta}}(r(s,a'))}
}
\,\partial_\beta \log \pi_{s,\beta}(a)
\right].
\label{eq:grpo_score}
\end{equation}
Then
\begin{equation}
\gamma_{s,\beta}
=
\frac{1}{\beta^2}
\left\|
\nabla^{\mathrm{nat}}J_s(\pi_{s,\beta})
\right\|_{F,\pi_{s,\beta}}
=
\frac{
\sqrt{\mathrm{Var}_{a\sim\pi_{s,\beta}}(r(s,a))}
}{\beta^2}.
\label{eq:pivot_theorem}
\end{equation}
In particular, at fixed $\beta$, states with larger reward variance induce a larger natural-gradient norm and a larger population GRPO score.
\end{theorem}

\paragraph{Interpretation.}
Theorem~\ref{thm:pivot} is a population, pathwise statement about the idealized KL-regularized statewise update $\pi_{s,\beta}\propto \pi_0 e^{r(s,\cdot)/\beta}$. It shows that, for group-normalized local RL, reward variance is not just a heuristic diagnostic: it is exactly the scale of the local natural-gradient signal along the KL path. For binary verifier rewards, larger variance means a more mixed success/failure turn, which is precisely why PivotRL filters toward mixed-outcome turns.

\begin{theorem}[Functional reward shifts mass toward acceptable actions with minimal KL change]
\label{thm:kl}
Fix $\beta>0$ and consider the regularized objective
\begin{equation}
\mathcal{L}_{\mathrm{func},\beta}(\pi)
=
\mathbb{E}_{s\sim d}
\left[
-
\mathbb{E}_{a\sim\pi(\cdot\mid s)}[r_{\mathrm{func}}(s,a)]
+
\beta\,\mathrm{KL}\!\bigl(\pi(\cdot\mid s)\,\|\,\pi_0(\cdot\mid s)\bigr)
\right],
\label{eq:kl_obj}
\end{equation}
where $r_{\mathrm{func}}(s,a)=\mathbf{1}[a\in \mathcal{M}(s)]$ and $d$ is a fixed state distribution. For each state $s$, define
\begin{equation}
\rho(s)=\pi_0(\mathcal{M}(s)\mid s),
\qquad
q_\beta(s)
=
\frac{\rho(s)e^{1/\beta}}{(1-\rho(s))+\rho(s)e^{1/\beta}}.
\label{eq:qbeta}
\end{equation}
Then $\mathcal{L}_{\mathrm{func},\beta}$ has a unique minimizer $\pi_\beta^*$ such that, for each state $s$ ($d$-almost surely),
\begin{equation}
\pi_\beta^*(\mathcal{M}(s)\mid s)=q_\beta(s)\ge \rho(s),
\label{eq:mass_shift}
\end{equation}
with strict inequality whenever $0<\rho(s)<1$. Moreover, among all distributions satisfying Eq.~\eqref{eq:mass_shift}, $\pi_\beta^*(\cdot\mid s)$ is the unique minimizer of
\[
\mathrm{KL}\!\bigl(\pi(\cdot\mid s)\,\|\,\pi_0(\cdot\mid s)\bigr),
\]
and it preserves the reference ordering within both $\mathcal{M}(s)$ and its complement $\mathcal{M}(s)^c$:
\begin{align}
\label{eq:conditional-distibution}
\frac{\pi_\beta^*(a\mid s)}{\pi_\beta^*(b\mid s)}
=
\frac{\pi_0(a\mid s)}{\pi_0(b\mid s)}
\quad
\text{for } a,b\in \mathcal{M}(s),
\qquad
\frac{\pi_\beta^*(a\mid s)}{\pi_\beta^*(b\mid s)}
=
\frac{\pi_0(a\mid s)}{\pi_0(b\mid s)}
\quad
\text{for } a,b\in \mathcal{M}(s)^c.
\end{align}
\end{theorem}

\paragraph{Interpretation.}
Theorem~\ref{thm:kl} isolates the effect of the functional reward, establishing that the minimizer policy of functional reward-based RL is the KL-projection of the reference policy onto the set of policies with a higher probability mass $q_\beta(s)$ on acceptable actions. From Eq.~\eqref{eq:conditional-distibution}, functional reward-based RL preserves the conditional distribution on both (i) the set of acceptable actions and (ii) its complement. Since the action space of assistant turns is exponentially large, any given action is generally relevant to only a single task. Under this assumption, the complement of acceptable actions corresponds to the set of task-unrelated actions, and consequently, the relative ranking among task-unrelated actions is preserved, explaining PivotRL's strong retention of OOD performance.

\section{Experiments}
\label{sec:experiments}

PivotRL achieves larger in-domain gains than same-data SFT while nearly eliminating OOD degradation, as shown in Section~\ref{sec:indomain_ood_sft}. On SWE-Bench, where end-to-end RL (E2E RL) is the standard training paradigm, PivotRL attains comparable accuracy without multi-turn environment rollouts, as shown in Section~\ref{sec:swe_e2e}. An ablation confirms that both pivot filtering and functional reward are necessary for the full gains, as shown in Section~\ref{sec:ablation_study}.

\paragraph{Setup.}
We train on four agentic domains separately---conversational tool use, software engineering, terminal control, and web browsing---and evaluate each resulting model on their corresponding benchmarks: $\tau^2$-Bench~\citep{barres2025tau2bench}, SWE-Bench Verified~\citep{jimenez2024swebench}, Terminal-Bench~\citep{tbench_2025}, and BrowseComp~\citep{wei2025browsecompsimplechallengingbenchmark}. All experiments start from \texttt{Qwen3-30B-A3B-Thinking-2507} (hereafter ``Base''), use Nemo-RL~\citep{nemo-rl} for optimization, and Nemo-Gym~\citep{nemo-gym} for environment rollouts. For every SFT--PivotRL comparison, the base model, prompts, and expert trajectories are identical. Domain-specific data construction, verifier design, and hyperparameters appear in Appendix~\ref{app:training_details}.

\subsection{In-Domain and OOD Accuracy}
\label{sec:indomain_ood_sft}

\begin{table}[t]
\centering
\small
\begin{tabular}{lcccc}
\toprule
\textbf{Benchmark} & \textbf{Base} & \textbf{SFT} & \textbf{PivotRL} & \textbf{$\Delta$ vs SFT} \\
\midrule
$\tau^2$-Bench   & 44.35 & 58.44 & 63.81 & +5.37 \\
SWE-Bench Verified     & 19.07 & 37.40 & 32.67 & -4.73 \\
Terminal-Bench   & 5.42  & 13.75 & 20.00 & +6.25 \\
BrowseComp       & 2.50  & 1.50  & 11.30 & +9.80 \\
\bottomrule
\end{tabular}
\caption{
In-domain results. Base denotes \texttt{Qwen3-30B-A3B-Thinking-2507}; SFT and PivotRL use identical training data. PivotRL outperforms SFT on three of four benchmarks, with an average gain of $+14.11$ over Base versus $+9.94$ for SFT.
}
\label{tab:single_domain_main_results}
\end{table}

Table~\ref{tab:single_domain_main_results} reports in-domain accuracy for each training domain. PivotRL improves over SFT on $\tau^2$-Bench ($+5.37$), Terminal-Bench ($+6.25$), and BrowseComp ($+9.80$), achieving an average in-domain gain of $+14.11$ over Base compared to $+9.94$ for SFT.
\begin{table}[t]
\centering
\small
\begin{tabular}{l c cc}
\toprule
\textbf{OOD Benchmark} & \textbf{Base} & \textbf{$\Delta$ SFT} & \textbf{$\Delta$ PivotRL} \\
\midrule
IFBench        & 52.04 & -11.46 & +0.82 \\
AIME25         & 86.04 & -19.72 & -1.20 \\
MATH500        & 98.05 & -8.51  & +0.35 \\
LiveCodeBench  & 66.52 & -7.76  & -0.17 \\
Scicode        & 36.83 & -11.39  & +1.90 \\
MMLU-Pro       & 80.84 & -2.99  & +0.31 \\
MMLU-ProX      & 78.58 & -7.16  & +0.12 \\
WMT24++        & 36.97 & -9.61  & -0.49 \\
\midrule
\textbf{Average} & 66.62 & -9.83 & +0.21 \\
\bottomrule
\end{tabular}
\caption{
OOD performance change ($\Delta$) relative to Base, averaged across the four training runs in Table~\ref{tab:single_domain_main_results}. SFT degrades by $-9.83$ on average; PivotRL maintains near-zero change ($+0.21$).
}
\label{tab:single_domain_ood_avg}
\end{table}
The more important result is OOD retention. Table~\ref{tab:single_domain_ood_avg} reports the average change on eight OOD benchmarks across the four training runs. SFT produces an average OOD change of $-9.83$, with the worst case after terminal-domain training (AIME25 drops by $-64.48$, MATH500 by $-34.50$). PivotRL stays near Base across all benchmarks, with an average change of $+0.21$ and no single benchmark dropping more than $-3.12$.
Table~\ref{tab:single_domain_ood_full} provides the full per-domain breakdown. In every training domain, PivotRL (rl) preserves OOD performance while SFT (sft) causes broad regression---most dramatically after terminal-domain training, where SFT drops AIME25 from $86.04$ to $21.56$.

\begin{table*}[t]
\centering
\small
\begin{tabular}{l c cc cc}
\toprule
\multirow{2}{*}{\textbf{Benchmark}}
  & \multirow{2}{*}{\textbf{Base}}
  & \multicolumn{2}{c}{$\tau^2$}
  & \multicolumn{2}{c}{SWE} \\
\cmidrule(lr){3-4}\cmidrule(lr){5-6}
  & & sft & rl & sft & rl \\
\midrule\midrule

\textbf{Peak In-Domain} & --
  & 58.44 (+14.09) & 63.81 (+19.46)
  & 37.40 (+18.33) & 32.67 (+13.60) \\
\midrule

IFBench    & 52.04
  & 44.42 (-7.62)  & 53.29 (+1.25)
  & 38.80 (-13.24) & 54.79 (+2.75) \\
AIME25     & 86.04
  & 76.25 (-9.79)  & 85.63 (-0.41)
  & 82.29 (-3.75)  & 85.00 (-1.04) \\
MATH500    & 98.05
  & 98.00 (-0.05)  & 98.35 (+0.30)
  & 98.30 (+0.25)  & 98.65 (+0.60) \\
LiveCodeBench & 66.52
  & 64.98 (-1.54)  & 66.96 (+0.44)
  & 57.27 (-9.25)  & 66.96 (+0.44) \\
Scicode    & 36.83
  & 31.95 (-4.88)  & 38.68 (+1.85)
  & 24.85 (-11.98) & 39.87 (+3.04) \\
MMLU-Pro   & 80.84
  & 79.81 (-1.03)  & 81.31 (+0.47)
  & 78.90 (-1.94)  & 81.13 (+0.29) \\
MMLU-ProX  & 78.58
  & 74.31 (-4.27)  & 78.92 (+0.34)
  & 75.58 (-3.00)  & 78.81 (+0.23) \\
WMT24++    & 36.97
  & 34.25 (-2.72)  & 36.54 (-0.43)
  & 35.69 (-1.28)  & 36.66 (-0.31) \\

\midrule\midrule

\multirow{2}{*}{\textbf{Benchmark}}
  & \multirow{2}{*}{\textbf{Base}}
  & \multicolumn{2}{c}{Terminal}
  & \multicolumn{2}{c}{BrowseComp} \\
\cmidrule(lr){3-4}\cmidrule(lr){5-6}
  & & sft & rl & sft & rl \\
\midrule\midrule

\textbf{Peak In-Domain} & --
  & 13.75 (+8.33)  & 20.00 (+14.58)
  & 1.50 (-1.00)   & 11.30 (+8.80) \\
\midrule

IFBench    & 52.04
  & 24.77 (-27.27) & 49.47 (-2.57)
  & 54.35 (+2.31)  & 53.89 (+1.85) \\
AIME25     & 86.04
  & 21.56 (-64.48) & 82.92 (-3.12)
  & 85.21 (-0.83)  & 85.83 (-0.21) \\
MATH500    & 98.05
  & 63.55 (-34.50) & 98.00 (-0.05)
  & 98.30 (+0.25)  & 98.40 (+0.35) \\
LiveCodeBench & 66.52
  & 45.15 (-21.37) & 64.54 (-1.98)
  & 67.62 (+1.10)  & 66.96 (+0.44) \\
Scicode    & 36.83
  & 9.39 (-27.44)  & 39.13 (+2.30)
  & 35.58 (-1.25)  & 37.28 (+0.45) \\
MMLU-Pro   & 80.84
  & 71.57 (-9.27)  & 80.85 (+0.01)
  & 81.13 (+0.29)  & 81.31 (+0.47) \\
MMLU-ProX  & 78.58
  & 59.13 (-19.45) & 78.49 (-0.09)
  & 76.68 (-1.90)  & 78.60 (+0.02) \\
WMT24++    & 36.97
  & 6.31 (-30.66)  & 36.48 (-0.49)
  & 33.18 (-3.79)  & 36.26 (-0.71) \\

\bottomrule
\end{tabular}
\caption{
Per-domain OOD breakdown across four single-domain training runs. Each cell shows score followed by $\Delta$ relative to Base. PivotRL (rl) consistently preserves OOD performance, whereas SFT (sft) causes broad regression---particularly after terminal-domain training.
}
\label{tab:single_domain_ood_full}
\end{table*}


\subsection{Comparison to End-to-End RL on SWE-Bench}
\label{sec:swe_e2e}

SWE-Bench is a natural comparison point because E2E RL is the standard training method for software-engineering agents: each GitHub issue requires a multi-turn tool-using trajectory, and the evaluation harness provides a binary success signal. To compare rollout cost, we count total rollout turns: for PivotRL, each training sample is a single-turn rollout, so the total equals the number of training samples; for E2E RL, we sum the turns across all training trajectories.

All methods start from Base and are evaluated with the OpenHands harness~\citep{wang2025openhandsopenplatformai}. Figure~\ref{fig:swe_speedup} plots accuracy against cumulative rollout turns and cumulative rollout time. To reach the same accuracy, PivotRL requires ${\sim}4\times$ fewer rollout turns and ${\sim}5.5\times$ less wall-clock time on the same number of compute nodes.

\begin{figure*}[t]
    \centering
    \includegraphics[width=\textwidth]{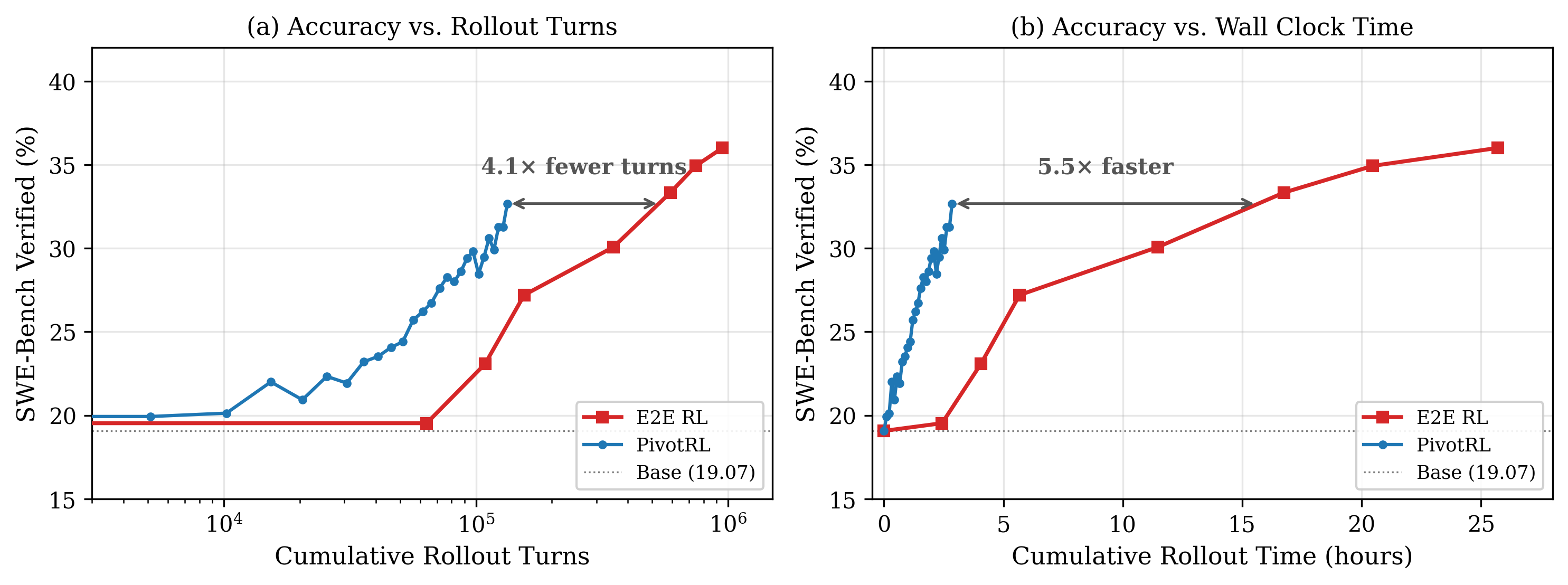}
    \caption{SWE-Bench accuracy vs.\ (a) cumulative rollout turns and (b) cumulative rollout time during training, using the same number of compute nodes. PivotRL reaches comparable accuracy to E2E RL with ${\sim}4\times$ fewer rollout turns and ${\sim}5.5\times$ less wall-clock time.}
    \label{fig:swe_speedup}
\end{figure*}

\subsection{Ablation Study}
\label{sec:ablation_study}

We present ablation results in Table \ref{tab:ablation_main}. To isolate the contribution of each PivotRL component, we remove one at a time on $\tau^2$-Bench.
\begin{table}[t]
\centering
\small
\begin{tabular}{lc}
\toprule
\textbf{Configuration} & \textbf{$\tau^2$-Bench} \\
\midrule
Full PivotRL ($\mathcal{D}_{\mathrm{adv}}$ + functional reward) & 63.81 \\
\quad $-$ Pivot filtering ($\mathcal{D}_{\mathrm{cand}}$ + functional reward) & 59.68 \\
\quad $-$ Functional reward ($\mathcal{D}_{\mathrm{cand}}$ + strict reward) & 57.34 \\
\midrule
Same-data SFT & 58.44 \\
Base & 44.35 \\
\bottomrule
\end{tabular}
\caption{
Ablation on $\tau^2$-Bench. Both pivot filtering and functional reward are necessary; removing either degrades accuracy.
}
\label{tab:ablation_main}
\end{table}
Removing filtering reduces accuracy from $63.81$ to $59.68$; removing functional reward yields $57.34$. Pivots concentrate rollouts on states with nonzero advantage signal (Proposition~\ref{prop:mixed}), and functional reward ensures that correct but textually different actions receive credit.

Figures~\ref{fig:tau_pivot_selection_acc} and~\ref{fig:tau_reward_std} show the training dynamics behind these gains. Under random sampling, per-batch reward variance collapses quickly, indicating that most sampled pivots no longer generate a useful advantage signal. The pivot sets preserve higher reward variance deeper into training and optimize to higher validation accuracy.

\begin{figure}[t]
    \centering
    \begin{minipage}{0.48\textwidth}
        \centering
        \includegraphics[width=\linewidth]{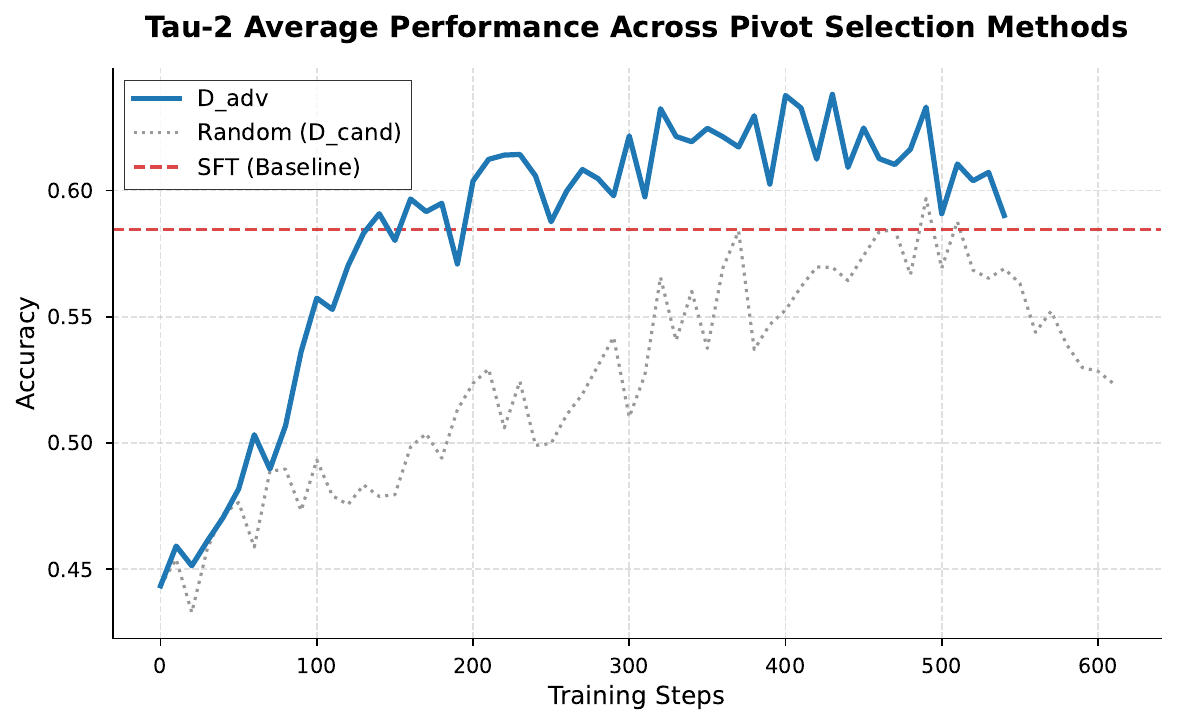}
        \caption{$\tau^2$-Bench accuracy over RL training. $\mathcal{D}_{adv}$ yields the strongest optimization and outperform same-data SFT.}
        \label{fig:tau_pivot_selection_acc}
    \end{minipage}
    \hfill
    \begin{minipage}{0.48\textwidth}
        \centering
        \includegraphics[width=\linewidth]{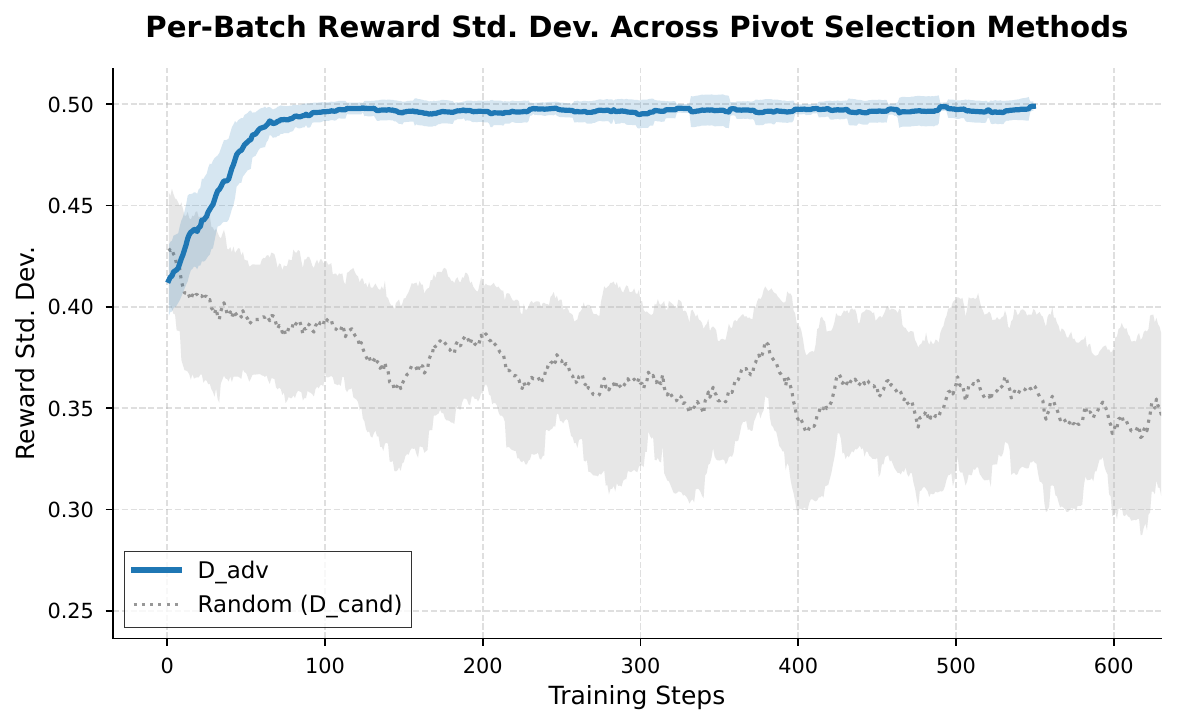}
        \caption{$\tau^2$-Bench reward standard deviation during RL training. $\mathcal{D}_{adv}$ sets maintain higher reward variance and therefore more informative local updates.}
        \label{fig:tau_reward_std}
    \end{minipage}
\end{figure}

\subsection{Integration into Large-Scale Post-Training}
\label{sec:large_scale}

PivotRL was used during the large scale post training of Nemotron-3-Super \citep{nvidia2025nemotron3super}. Table~\ref{tab:large_scale_pivot} reports agentic benchmark accuracy during the Nemotron-3-Super post-training pipeline, where PivotRL covers the agentic environments while other RL environments handle reasoning and chat. 

\begin{table}[t]
\centering
\small
\begin{tabular}{lcc}
\toprule
\textbf{Benchmark} & \textbf{Nemotron-3-Super SFT} & \textbf{After PivotRL Stage} \\
\midrule
$\tau^2$-Bench         & 48.00 & 64.00 \\
SWE-Bench Verified     & 12.87 & 61.33 \\
Terminal-Bench 1.1 Core & 23.33 & 34.17 \\
BrowseComp             & 13.03 & 25.04 \\
\bottomrule
\end{tabular}
\caption{
Nemotron-3-Super agentic benchmark accuracy before and after the RL stage that included PivotRL environments for the agentic verticals and other RL environments for reasoning and chat. 
}
\label{tab:large_scale_pivot}
\end{table}

\section{Related Works}
\label{sec:related}

\subsection{Agentic LLMs and Training Recipes}
Agentic language models interleave natural language with grounded actions for complex environment interactions across tools, code, and web navigation \citep{yao2023react,schick2023toolformer,xu2023toolbench,qin2023toolllm,shridhar2020alfworld,wang2022scienceworld,jimenez2024swebench,wang2025openhandsopenplatformai,pan2025trainingsoftwareengineeringagents,yao2022webshop,wei2025browsecompsimplechallengingbenchmark,ma2024agentboard,wang2024sotopiapi}. 
Reinforcement learning (RL) further optimizes multi-turn exploration and credit assignment \citep{ouyang2022training,schulman2017ppo,schulman2015trpo,wang2025sparl,song2024trialerror,zhang2025rlvmr}, supported by recent advances in scalable policy optimization, hierarchical modeling, and multi-agent coordination \citep{ji2025treesearchllmrl,rohatgi2025tamingprocessverifiers,xue2025simpletir,dong2025toolstar,yang2025dcpo,yu2025dapoopensourcellmreinforcement,luo2025agentlightningtrainai,dong2025agenticreinforcedpolicyoptimization,su2025toolorchestra,zhou2024archer,prabhakar2025apigenmtagenticpipelinemultiturn,zhou2025sweetrl,xi2025agentgymrl}. 
In parallel, RL with verifiable rewards (RLVR) succeeds in reasoning domains \citep{shao2024deepseekmath,xin2024deepseekproverv15harnessingproofassistant,Guo_2025,5team2025glm45agenticreasoningcoding}, while on-policy distillation mitigates off-policy distribution shifts \citep{lu2025onpolicydistillation,zhao2026selfdistilledreasoneronpolicyselfdistillation,xu2025learninginteractionagenticdistillation,hubotter2026selfdistillation}. 
Our work lies at this intersection: we target agentic tasks by converting supervised trajectories into turn-level verifier rewards, enabling on-policy optimization without a separate reward model.

\subsection{Adapting Behavior Cloning into RL}
Offline supervised fine-tuning (SFT) frequently suffers from catastrophic forgetting and degraded generalization \citep{luo2023empiricalCF,li2024revisitingCF,imanov2026mechanisticCF,kirkpatrick2017ewc,li2016lwf}, aligning with the superficial alignment hypothesis \citep{zhou2023limaalignment,vergarabrowne2026sahcomplexity,raghavendra2024revisitingSAH}. Theoretically, offline imitation suboptimality grows quadratically with task horizon \citep{rajaraman2020towardfundamentallimitsimitation}, necessitating environment interaction via online data aggregation to overcome \citep{ross2011dagger,swamy2022minimaxoptimalonlineimitation}. 
Crucially, on-policy RL mitigates this forgetting by retaining broader capabilities \citep{shenfeld2025rlsrazoronlinereinforcement,chen2025retainingdoingroleonpolicy}. To transfer offline knowledge online, prior methods combine expert demonstrations with RL updates \citep{uchendu2022jumpstartreinforcementlearning,hester2018dqfd}, utilize partial trace prefixes \citep{setlur2026prefixrl}, or repurpose SFT tokens as rollout rewards \citep{ming2025onetokenrollout}. Building on this SFT-to-RL bridge, our approach specializes in multi-turn tool use via turn-based pivot selection and functional rewards.
\section{Conclusion}
\label{sec:conclusion}

In this work, we introduce PivotRL, a turn-level reinforcement learning algorithm designed to tractably post-train large language models for long-horizon agentic tasks. 
Our empirical evaluations across four diverse agentic domains demonstrate the effectiveness and scalability of this approach. Compared to standard SFT on identical data, PivotRL achieves a $+4.17$ higher in-domain accuracy on average, alongside a substantial $+10.04$ out-of-domain (OOD) retention advantage that heavily mitigates catastrophic forgetting. Furthermore, when compared to end-to-end RL, PivotRL achieves comparable accuracy at a fraction of the compute cost.
PivotRL's performance gain comes from better action-space coverage obtained by on-policy rollouts and verifier-derived rewards.
In future work, we plan to extend our framework to incorporate non-programmatic verifiers, such as LLM-as-a-judge frameworks~\citep{zheng2023judgingllmasajudgemtbenchchatbot} and process reward models~\citep{lightman2023letsverifystepstep}, as well as exploring online reward profiling approaches such as dynamic sampling~\citep{yu2025dapoopensourcellmreinforcement}.

\bibliography{references}

\begin{thebibliography}{73}
\providecommand{\natexlab}[1]{#1}
\providecommand{\url}[1]{\texttt{#1}}
\expandafter\ifx\csname urlstyle\endcsname\relax
  \providecommand{\doi}[1]{doi: #1}\else
  \providecommand{\doi}{doi: \begingroup \urlstyle{rm}\Url}\fi

\bibitem[Anomaly(2026)]{opencode}
Anomaly.
\newblock Opencode: The open source coding agent, 2026.
\newblock URL \url{https://github.com/anomalyco/opencode}.

\bibitem[Barres et~al.(2025)Barres, Dong, Ray, Si, and Narasimhan]{barres2025tau2bench}
Victor Barres, Honghua Dong, Soham Ray, Xujie Si, and Karthik Narasimhan.
\newblock $\tau^2$-bench: Evaluating conversational agents in a dual-control environment, 2025.
\newblock URL \url{https://arxiv.org/abs/2506.07982}.

\bibitem[Chen et~al.(2025)Chen, Razin, Narasimhan, and Chen]{chen2025retainingdoingroleonpolicy}
Howard Chen, Noam Razin, Karthik Narasimhan, and Danqi Chen.
\newblock Retaining by doing: The role of on-policy data in mitigating forgetting, 2025.
\newblock URL \url{https://arxiv.org/abs/2510.18874}.

\bibitem[Chu et~al.(2025)Chu, Zhai, Yang, Tong, Xie, Schuurmans, Le, Levine, and Ma]{chu2025sftmemorizesrlgeneralizes}
Tianzhe Chu, Yuexiang Zhai, Jihan Yang, Shengbang Tong, Saining Xie, Dale Schuurmans, Quoc~V. Le, Sergey Levine, and Yi~Ma.
\newblock Sft memorizes, rl generalizes: A comparative study of foundation model post-training, 2025.
\newblock URL \url{https://arxiv.org/abs/2501.17161}.

\bibitem[DeepSeek-AI(2025{\natexlab{a}})]{Guo_2025}
DeepSeek-AI.
\newblock Deepseek-r1 incentivizes reasoning in llms through reinforcement learning.
\newblock \emph{Nature}, 645\penalty0 (8081):\penalty0 633–638, September 2025{\natexlab{a}}.
\newblock ISSN 1476-4687.
\newblock \doi{10.1038/s41586-025-09422-z}.
\newblock URL \url{http://dx.doi.org/10.1038/s41586-025-09422-z}.

\bibitem[DeepSeek-AI(2025{\natexlab{b}})]{deepseekai2025deepseekv32pushingfrontieropen}
DeepSeek-AI.
\newblock Deepseek-v3.2: Pushing the frontier of open large language models, 2025{\natexlab{b}}.
\newblock URL \url{https://arxiv.org/abs/2512.02556}.

\bibitem[Dong et~al.(2025{\natexlab{a}})Dong, Chen, Li, Jin, Qian, Zhu, Mao, Zhou, Dou, and Wen]{dong2025toolstar}
Guanting Dong, Yifei Chen, Xiaoxi Li, Jiajie Jin, Hongjin Qian, Yutao Zhu, Hangyu Mao, Guorui Zhou, Zhicheng Dou, and Ji-Rong Wen.
\newblock Tool-star: Empowering {LLM}-brained multi-tool reasoner via reinforcement learning, 2025{\natexlab{a}}.
\newblock URL \url{https://arxiv.org/abs/2505.16410}.

\bibitem[Dong et~al.(2025{\natexlab{b}})Dong, Mao, Ma, Bao, Chen, Wang, Chen, Du, Wang, Zhang, Zhou, Zhu, Wen, and Dou]{dong2025agenticreinforcedpolicyoptimization}
Guanting Dong, Hangyu Mao, Kai Ma, Licheng Bao, Yifei Chen, Zhongyuan Wang, Zhongxia Chen, Jiazhen Du, Huiyang Wang, Fuzheng Zhang, Guorui Zhou, Yutao Zhu, Ji-Rong Wen, and Zhicheng Dou.
\newblock Agentic reinforced policy optimization, 2025{\natexlab{b}}.
\newblock URL \url{https://arxiv.org/abs/2507.19849}.

\bibitem[Hester et~al.(2018)Hester, Vecerik, Pietquin, Lanctot, Schaul, Piot, Horgan, Quan, Sendonaris, Osband, Dulac-Arnold, Agapiou, Leibo, and Gruslys]{hester2018dqfd}
Todd Hester, Matej Vecerik, Olivier Pietquin, Marc Lanctot, Tom Schaul, Bilal Piot, Dan Horgan, John Quan, Andrew Sendonaris, Ian Osband, Gabriel Dulac-Arnold, John Agapiou, Joel~Z. Leibo, and Audrunas Gruslys.
\newblock Deep q-learning from demonstrations.
\newblock In \emph{Proceedings of the Thirty-Second AAAI Conference on Artificial Intelligence and Thirtieth Innovative Applications of Artificial Intelligence Conference and Eighth AAAI Symposium on Educational Advances in Artificial Intelligence}, AAAI'18/IAAI'18/EAAI'18. AAAI Press, 2018.
\newblock ISBN 978-1-57735-800-8.

\bibitem[H{\"u}botter et~al.(2026)H{\"u}botter, L{\"u}beck, Behric, Baumann, Bagatella, Marta, Hakimi, Shenfeld, Buening, Guestrin, and Krause]{hubotter2026selfdistillation}
Jonas H{\"u}botter, Frederike L{\"u}beck, Lejs Behric, Anton Baumann, Marco Bagatella, Daniel Marta, Ido Hakimi, Idan Shenfeld, Thomas~Kleine Buening, Carlos Guestrin, and Andreas Krause.
\newblock Reinforcement learning via self-distillation, 2026.
\newblock URL \url{https://arxiv.org/abs/2601.20802}.

\bibitem[Imanov(2026)]{imanov2026mechanisticCF}
Olaf Yunus~Laitinen Imanov.
\newblock Mechanistic analysis of catastrophic forgetting in large language models during continual fine-tuning, 2026.
\newblock URL \url{https://arxiv.org/abs/2601.18699}.

\bibitem[Jain et~al.(2025)Jain, Singh, Shetty, Zheng, Sen, and Stoica]{jain2025r2egymproceduralenvironmentshybrid}
Naman Jain, Jaskirat Singh, Manish Shetty, Liang Zheng, Koushik Sen, and Ion Stoica.
\newblock R2e-gym: Procedural environments and hybrid verifiers for scaling open-weights swe agents, 2025.
\newblock URL \url{https://arxiv.org/abs/2504.07164}.

\bibitem[Ji et~al.(2025)Ji, Ma, Wang, Chen, Chu, and Wu]{ji2025treesearchllmrl}
Yuxiang Ji, Ziyu Ma, Yong Wang, Guanhua Chen, Xiangxiang Chu, and Liaoni Wu.
\newblock Tree search for {LLM} agent reinforcement learning, 2025.
\newblock URL \url{https://arxiv.org/abs/2509.21240}.

\bibitem[Jimenez et~al.(2024)Jimenez, Yang, Wettig, Yao, Pei, Press, and Narasimhan]{jimenez2024swebench}
Carlos~E Jimenez, John Yang, Alexander Wettig, Shunyu Yao, Kexin Pei, Ofir Press, and Karthik~R Narasimhan.
\newblock {SWE}-bench: Can language models resolve real-world github issues?
\newblock In \emph{The Twelfth International Conference on Learning Representations}, 2024.
\newblock URL \url{https://openreview.net/forum?id=VTF8yNQM66}.

\bibitem[Kirkpatrick et~al.(2017)Kirkpatrick, Pascanu, Rabinowitz, Veness, Desjardins, Rusu, Milan, Quan, Ramalho, Grabska-Barwinska, Hassabis, Clopath, Kumaran, and Hadsell]{kirkpatrick2017ewc}
James Kirkpatrick, Razvan Pascanu, Neil Rabinowitz, Joel Veness, Guillaume Desjardins, Andrei~A. Rusu, Kieran Milan, John Quan, Tiago Ramalho, Agnieszka Grabska-Barwinska, Demis Hassabis, Claudia Clopath, Dharshan Kumaran, and Raia Hadsell.
\newblock Overcoming catastrophic forgetting in neural networks.
\newblock \emph{Proceedings of the National Academy of Sciences}, 114\penalty0 (13):\penalty0 3521--3526, 2017.
\newblock \doi{10.1073/pnas.1611835114}.
\newblock URL \url{https://www.pnas.org/doi/abs/10.1073/pnas.1611835114}.

\bibitem[Li et~al.(2024)Li, Ding, Fang, and Tao]{li2024revisitingCF}
Hongyu Li, Liang Ding, Meng Fang, and Dacheng Tao.
\newblock Revisiting catastrophic forgetting in large language model tuning.
\newblock \emph{arXiv preprint arXiv:2406.04836}, 2024.
\newblock \doi{10.48550/arXiv.2406.04836}.
\newblock URL \url{https://arxiv.org/abs/2406.04836}.

\bibitem[Li \& Hoiem(2016)Li and Hoiem]{li2016lwf}
Zhizhong Li and Derek Hoiem.
\newblock Learning without forgetting.
\newblock In \emph{Proceedings of the European Conference on Computer Vision (ECCV)}, 2016.
\newblock URL \url{https://arxiv.org/abs/1606.09282}.

\bibitem[Lightman et~al.(2023)Lightman, Kosaraju, Burda, Edwards, Baker, Lee, Leike, Schulman, Sutskever, and Cobbe]{lightman2023letsverifystepstep}
Hunter Lightman, Vineet Kosaraju, Yura Burda, Harri Edwards, Bowen Baker, Teddy Lee, Jan Leike, John Schulman, Ilya Sutskever, and Karl Cobbe.
\newblock Let's verify step by step, 2023.
\newblock URL \url{https://arxiv.org/abs/2305.20050}.

\bibitem[Liu et~al.(2025{\natexlab{a}})Liu, Luo, Mao, Chen, Li, Vu, and Haffari]{xu2025learninginteractionagenticdistillation}
Junnan Liu, Linhao Luo, Qianren Mao, Zhijun Chen, Zhuoran Li, Thuy-Trang Vu, and Gholamreza Haffari.
\newblock Learning with interaction: Agentic distillation for large language model reasoning.
\newblock In \emph{International Conference on Learning Representations}, 2025{\natexlab{a}}.
\newblock URL \url{https://openreview.net/forum?id=zyp9QT5Gf1}.

\bibitem[Liu et~al.(2025{\natexlab{b}})Liu, Huang, Zeng, Hao, Yu, Li, Wang, Gan, Liu, Yu, Wang, Wang, Ning, Hou, Wang, Wu, Wang, Liu, Wang, Tang, Tu, Shang, Jiang, Tang, Lian, Liu, and Chen]{liu2025toolacewinningpointsllm}
Weiwen Liu, Xu~Huang, Xingshan Zeng, Xinlong Hao, Shuai Yu, Dexun Li, Shuai Wang, Weinan Gan, Zhengying Liu, Yuanqing Yu, Zezhong Wang, Yuxian Wang, Wu~Ning, Yutai Hou, Bin Wang, Chuhan Wu, Xinzhi Wang, Yong Liu, Yasheng Wang, Duyu Tang, Dandan Tu, Lifeng Shang, Xin Jiang, Ruiming Tang, Defu Lian, Qun Liu, and Enhong Chen.
\newblock Toolace: Winning the points of llm function calling, 2025{\natexlab{b}}.
\newblock URL \url{https://arxiv.org/abs/2409.00920}.

\bibitem[Lu \& Lab(2025)Lu and Lab]{lu2025onpolicydistillation}
Kevin Lu and Thinking~Machines Lab.
\newblock On-policy distillation.
\newblock \emph{Thinking Machines Lab: Connectionism}, 2025.
\newblock URL \url{https://thinkingmachines.ai/blog/on-policy-distillation}.

\bibitem[Luo et~al.(2025{\natexlab{a}})Luo, Zhang, He, Wang, Zhao, Li, Qiu, and Yang]{luo2025agentlightningtrainai}
Xufang Luo, Yuge Zhang, Zhiyuan He, Zilong Wang, Siyun Zhao, Dongsheng Li, Luna~K. Qiu, and Yuqing Yang.
\newblock Agent lightning: Train any ai agents with reinforcement learning, 2025{\natexlab{a}}.
\newblock URL \url{https://arxiv.org/abs/2508.03680}.

\bibitem[Luo et~al.(2025{\natexlab{b}})Luo, Yang, Meng, Li, Zhou, and Zhang]{luo2023empiricalCF}
Yun Luo, Zhen Yang, Fandong Meng, Yafu Li, Jie Zhou, and Yue Zhang.
\newblock An empirical study of catastrophic forgetting in large language models during continual fine-tuning, 2025{\natexlab{b}}.
\newblock URL \url{https://arxiv.org/abs/2308.08747}.

\bibitem[Ma et~al.(2024)Ma, Zhang, Zhu, Yang, Yang, Jin, Lan, Kong, and He]{ma2024agentboard}
Chang Ma, Junlei Zhang, Zhihao Zhu, Cheng Yang, Yujiu Yang, Yaohui Jin, Zhenzhong Lan, Lingpeng Kong, and Junxian He.
\newblock Agentboard: An analytical evaluation board of multi-turn {LLM} agents, 2024.
\newblock URL \url{https://arxiv.org/abs/2401.13178}.

\bibitem[Ming et~al.(2026)Ming, Wu, Hu, He, and Yu]{ming2025onetokenrollout}
Rui Ming, Haoyuan Wu, Shoubo Hu, Zhuolun He, and Bei Yu.
\newblock One-token rollout: Guiding supervised fine-tuning of llms with policy gradient, 2026.
\newblock URL \url{https://arxiv.org/abs/2509.26313}.

\bibitem[{MiniMax}(2026)]{minimax2026m25}
{MiniMax}.
\newblock Minimax-m2.5, 2026.
\newblock URL \url{https://huggingface.co/MiniMaxAI/MiniMax-M2.5}.

\bibitem[NVIDIA(2025{\natexlab{a}})]{nemo-gym}
NVIDIA.
\newblock Nemo gym: An open source library for scaling reinforcement learning environments for llm.
\newblock \url{https://github.com/NVIDIA-NeMo/Gym}, 2025{\natexlab{a}}.
\newblock GitHub repository.

\bibitem[NVIDIA(2025{\natexlab{b}})]{nemo-rl}
NVIDIA.
\newblock Nemo rl: A scalable and efficient post-training library.
\newblock \url{https://github.com/NVIDIA-NeMo/RL}, 2025{\natexlab{b}}.
\newblock GitHub repository.

\bibitem[{OpenAI}(2026)]{openaicodex_software}
{OpenAI}.
\newblock Openai codex: The coding agent for software development, 2026.
\newblock URL \url{https://developers.openai.com/codex}.

\bibitem[Ouyang et~al.(2022)Ouyang, Wu, Jiang, Almeida, Wainwright, Mishkin, Zhang, Agarwal, Slama, Ray, Schulman, Hilton, Kelton, Miller, Simens, Askell, Welinder, Christiano, Leike, and Lowe]{ouyang2022training}
Long Ouyang, Jeff Wu, Xu~Jiang, Diogo Almeida, Carroll~L. Wainwright, Pamela Mishkin, Chong Zhang, Sandhini Agarwal, Katarina Slama, Alex Ray, John Schulman, Jacob Hilton, Fraser Kelton, Luke Miller, Maddie Simens, Amanda Askell, Peter Welinder, Paul Christiano, Jan Leike, and Ryan Lowe.
\newblock Training language models to follow instructions with human feedback, 2022.
\newblock URL \url{https://arxiv.org/abs/2203.02155}.

\bibitem[Pan et~al.(2025)Pan, Wang, Neubig, Jaitly, Ji, Suhr, and Zhang]{pan2025trainingsoftwareengineeringagents}
Jiayi Pan, Xingyao Wang, Graham Neubig, Navdeep Jaitly, Heng Ji, Alane Suhr, and Yizhe Zhang.
\newblock Training software engineering agents and verifiers with swe-gym, 2025.
\newblock URL \url{https://arxiv.org/abs/2412.21139}.

\bibitem[Prabhakar et~al.(2025)Prabhakar, Liu, Zhu, Zhang, Awalgaonkar, Wang, Liu, Chen, Hoang, Niebles, Heinecke, Yao, Wang, Savarese, and Xiong]{prabhakar2025apigenmtagenticpipelinemultiturn}
Akshara Prabhakar, Zuxin Liu, Ming Zhu, Jianguo Zhang, Tulika Awalgaonkar, Shiyu Wang, Zhiwei Liu, Haolin Chen, Thai Hoang, Juan~Carlos Niebles, Shelby Heinecke, Weiran Yao, Huan Wang, Silvio Savarese, and Caiming Xiong.
\newblock Apigen-mt: Agentic pipeline for multi-turn data generation via simulated agent-human interplay, 2025.
\newblock URL \url{https://arxiv.org/abs/2504.03601}.

\bibitem[Qin et~al.(2023)Qin, Liang, Ye, Zhu, Yan, Lu, Lin, Cong, Tang, Qian, Zhao, Hong, Tian, Xie, Zhou, Gerstein, Li, Liu, and Sun]{qin2023toolllm}
Yujia Qin, Shihao Liang, Yining Ye, Kunlun Zhu, Lan Yan, Yaxi Lu, Yankai Lin, Xin Cong, Xiangru Tang, Bill Qian, Sihan Zhao, Lauren Hong, Runchu Tian, Ruobing Xie, Jie Zhou, Mark Gerstein, Dahai Li, Zhiyuan Liu, and Maosong Sun.
\newblock Toolllm: Facilitating large language models to master 16000+ real-world apis, 2023.
\newblock URL \url{https://arxiv.org/abs/2307.16789}.

\bibitem[Raghavendra et~al.(2024)Raghavendra, Nath, and Hendryx]{raghavendra2024revisitingSAH}
Mohit Raghavendra, Vaskar Nath, and Sean Hendryx.
\newblock Revisiting the superficial alignment hypothesis, 2024.
\newblock URL \url{https://arxiv.org/abs/2410.03717}.

\bibitem[Rajaraman et~al.(2020)Rajaraman, Yang, Jiao, and Ramchandran]{rajaraman2020towardfundamentallimitsimitation}
Nived Rajaraman, Lin~F. Yang, Jiantao Jiao, and Kannan Ramchandran.
\newblock Toward the fundamental limits of imitation learning.
\newblock \emph{Advances in Neural Information Processing Systems}, 33:\penalty0 2914--2924, 2020.

\bibitem[Rohatgi et~al.(2025)Rohatgi, Shetty, Saless, Li, Moitra, Risteski, and Foster]{rohatgi2025tamingprocessverifiers}
Dhruv Rohatgi, Abhishek Shetty, Donya Saless, Yuchen Li, Ankur Moitra, Andrej Risteski, and Dylan~J. Foster.
\newblock Taming imperfect process verifiers: A sampling perspective on backtracking, 2025.
\newblock URL \url{https://arxiv.org/abs/2510.03149}.

\bibitem[Ross et~al.(2011)Ross, Gordon, and Bagnell]{ross2011dagger}
St{\'e}phane Ross, Geoffrey~J. Gordon, and J.~Andrew Bagnell.
\newblock A reduction of imitation learning and structured prediction to no-regret online learning.
\newblock In \emph{Proceedings of the Fourteenth International Conference on Artificial Intelligence and Statistics (AISTATS)}, 2011.
\newblock URL \url{http://proceedings.mlr.press/v15/ross11a.html}.

\bibitem[Schick et~al.(2023)Schick, Dwivedi-Yu, Dessì, Raileanu, Lomeli, Zettlemoyer, Cancedda, and Scialom]{schick2023toolformer}
Timo Schick, Jane Dwivedi-Yu, Roberto Dessì, Roberta Raileanu, Maria Lomeli, Luke Zettlemoyer, Nicola Cancedda, and Thomas Scialom.
\newblock Toolformer: Language models can teach themselves to use tools, 2023.
\newblock URL \url{https://arxiv.org/abs/2302.04761}.

\bibitem[Schulman et~al.(2015)Schulman, Levine, Abbeel, Jordan, and Moritz]{schulman2015trpo}
John Schulman, Sergey Levine, Pieter Abbeel, Michael Jordan, and Philipp Moritz.
\newblock Trust region policy optimization.
\newblock In Francis Bach and David Blei (eds.), \emph{Proceedings of the 32nd International Conference on Machine Learning}, volume~37 of \emph{Proceedings of Machine Learning Research}, pp.\  1889--1897, Lille, France, 07--09 Jul 2015. PMLR.
\newblock URL \url{https://proceedings.mlr.press/v37/schulman15.html}.

\bibitem[Schulman et~al.(2017)Schulman, Wolski, Dhariwal, Radford, and Klimov]{schulman2017ppo}
John Schulman, Filip Wolski, Prafulla Dhariwal, Alec Radford, and Oleg Klimov.
\newblock Proximal policy optimization algorithms, 2017.
\newblock URL \url{https://arxiv.org/abs/1707.06347}.

\bibitem[Setlur et~al.(2026)Setlur, Wang, Cohen, Rashidinejad, and Xie]{setlur2026prefixrl}
Amrith Setlur, Zijian Wang, Andrew Cohen, Paria Rashidinejad, and Sang~Michael Xie.
\newblock Reuse your flops: Scaling rl on hard problems by conditioning on very off-policy prefixes, 2026.
\newblock URL \url{https://arxiv.org/abs/2601.18795}.

\bibitem[Shao et~al.(2024)Shao, Wang, Zhu, Xu, Song, Bi, Zhang, Zhang, Li, Wu, and Guo]{shao2024deepseekmath}
Zhihong Shao, Peiyi Wang, Qihao Zhu, Runxin Xu, Junxiao Song, Xiao Bi, Haowei Zhang, Mingchuan Zhang, Y.K. Li, Y.~Wu, and Daya Guo.
\newblock Deepseekmath: Pushing the limits of mathematical reasoning in open language models.
\newblock \emph{arXiv preprint arXiv:2402.03300}, 2024.

\bibitem[Shenfeld et~al.(2025)Shenfeld, Pari, and Agrawal]{shenfeld2025rlsrazoronlinereinforcement}
Idan Shenfeld, Jyothish Pari, and Pulkit Agrawal.
\newblock Rl's razor: Why online reinforcement learning forgets less, 2025.
\newblock URL \url{https://arxiv.org/abs/2509.04259}.

\bibitem[Shridhar et~al.(2021)Shridhar, Yuan, Côté, Bisk, Trischler, and Hausknecht]{shridhar2020alfworld}
Mohit Shridhar, Xingdi Yuan, Marc-Alexandre Côté, Yonatan Bisk, Adam Trischler, and Matthew Hausknecht.
\newblock Alfworld: Aligning text and embodied environments for interactive learning, 2021.
\newblock URL \url{https://arxiv.org/abs/2010.03768}.

\bibitem[Song et~al.(2024)Song, Yin, Yue, Huang, Li, and Lin]{song2024trialerror}
Yifan Song, Da~Yin, Xiang Yue, Jie Huang, Sujian Li, and Bill~Yuchen Lin.
\newblock Trial and error: Exploration-based trajectory optimization for {LLM} agents, 2024.
\newblock URL \url{https://arxiv.org/abs/2403.02502}.

\bibitem[Su et~al.(2025)Su, Diao, Lu, Liu, Xu, Dong, Fu, Belcak, Ye, Yin, Dong, Bakhturina, Yu, Choi, Kautz, and Molchanov]{su2025toolorchestra}
Hongjin Su, Shizhe Diao, Ximing Lu, Mingjie Liu, Jiacheng Xu, Xin Dong, Yonggan Fu, Peter Belcak, Hanrong Ye, Hongxu Yin, Yi~Dong, Evelina Bakhturina, Tao Yu, Yejin Choi, Jan Kautz, and Pavlo Molchanov.
\newblock Toolorchestra: Elevating intelligence via efficient model and tool orchestration, 2025.
\newblock URL \url{https://arxiv.org/abs/2511.21689}.

\bibitem[Swamy et~al.(2023)Swamy, Rajaraman, Peng, Choudhury, Bagnell, Wu, Jiao, and Ramchandran]{swamy2022minimaxoptimalonlineimitation}
Gokul Swamy, Nived Rajaraman, Matthew Peng, Sanjiban Choudhury, J.~Andrew Bagnell, Zhiwei~Steven Wu, Jiantao Jiao, and Kannan Ramchandran.
\newblock Minimax optimal online imitation learning via replay estimation, 2023.
\newblock URL \url{https://arxiv.org/abs/2205.15397}.

\bibitem[Team(2025{\natexlab{a}})]{5team2025glm45agenticreasoningcoding}
GLM-4.5 Team.
\newblock Glm-4.5: Agentic, reasoning, and coding (arc) foundation models, 2025{\natexlab{a}}.
\newblock URL \url{https://arxiv.org/abs/2508.06471}.

\bibitem[Team(2025{\natexlab{b}})]{kimiteam2025kimik2openagentic}
Kimi Team.
\newblock Kimi k2: Open agentic intelligence, 2025{\natexlab{b}}.
\newblock URL \url{https://arxiv.org/abs/2507.20534}.

\bibitem[Team(2026)]{nvidia2025nemotron3super}
NVIDIA~Nemotron Team.
\newblock Nvidia nemotron-3 super technical report.
\newblock \url{https://research.nvidia.com/labs/nemotron/files/NVIDIA-Nemotron-3-Super-Technical-Report.pdf}, 2026.
\newblock NVIDIA Research Technical Report.

\bibitem[Team(2025{\natexlab{c}})]{tbench_2025}
The Terminal-Bench Team.
\newblock Terminal-bench: A benchmark for ai agents in terminal environments, Apr 2025{\natexlab{c}}.
\newblock URL \url{https://github.com/laude-institute/terminal-bench}.

\bibitem[Uchendu et~al.(2023)Uchendu, Xiao, Lu, Zhu, Yan, Simon, Bennice, Fu, Ma, Jiao, Levine, and Hausman]{uchendu2022jumpstartreinforcementlearning}
Ikechukwu Uchendu, Ted Xiao, Yao Lu, Banghua Zhu, Mengyuan Yan, Jos{\'e}phine Simon, Matthew Bennice, Chuyuan Fu, Cong Ma, Jiantao Jiao, Sergey Levine, and Karol Hausman.
\newblock Jump-start reinforcement learning.
\newblock In \emph{International Conference on Machine Learning}, 2023.

\bibitem[Vergara-Browne et~al.(2026)Vergara-Browne, Patil, Titov, Reddy, Pimentel, and Mosbach]{vergarabrowne2026sahcomplexity}
Tomás Vergara-Browne, Darshan Patil, Ivan Titov, Siva Reddy, Tiago Pimentel, and Marius Mosbach.
\newblock Operationalising the superficial alignment hypothesis via task complexity, 2026.
\newblock URL \url{https://arxiv.org/abs/2602.15829}.

\bibitem[Wang et~al.(2025{\natexlab{a}})Wang, Leong, Wang, Wang, and Li]{wang2025sparl}
Hanlin Wang, Chak~Tou Leong, Jiashuo Wang, Jian Wang, and Wenjie Li.
\newblock Spa-rl: Reinforcing {LLM} agents via stepwise progress attribution, 2025{\natexlab{a}}.
\newblock URL \url{https://arxiv.org/abs/2505.20732}.

\bibitem[Wang et~al.(2024)Wang, Yu, Zhang, Qi, Sap, Neubig, Bisk, and Zhu]{wang2024sotopiapi}
Ruiyi Wang, Haofei Yu, Wenxin Zhang, Zhengyang Qi, Maarten Sap, Graham Neubig, Yonatan Bisk, and Hao Zhu.
\newblock Sotopia-$\pi$: Interactive learning of socially intelligent language agents, 2024.
\newblock URL \url{https://arxiv.org/abs/2403.08715}.

\bibitem[Wang et~al.(2022)Wang, Jansen, C{\^o}t{\'e}, and Ammanabrolu]{wang2022scienceworld}
Ruoyao Wang, Peter Jansen, Marc-Alexandre C{\^o}t{\'e}, and Prithviraj Ammanabrolu.
\newblock Scienceworld: Is your agent smarter than a 5th grader?, 2022.
\newblock URL \url{https://arxiv.org/abs/2203.07540}.

\bibitem[Wang et~al.(2025{\natexlab{b}})Wang, Li, Song, Xu, Tang, Zhuge, Pan, Song, Li, Singh, Tran, Li, Ma, Zheng, Qian, Shao, Muennighoff, Zhang, Hui, Lin, Brennan, Peng, Ji, and Neubig]{wang2025openhandsopenplatformai}
Xingyao Wang, Boxuan Li, Yufan Song, Frank~F. Xu, Xiangru Tang, Mingchen Zhuge, Jiayi Pan, Yueqi Song, Bowen Li, Jaskirat Singh, Hoang~H. Tran, Fuqiang Li, Ren Ma, Mingzhang Zheng, Bill Qian, Yanjun Shao, Niklas Muennighoff, Yizhe Zhang, Binyuan Hui, Junyang Lin, Robert Brennan, Hao Peng, Heng Ji, and Graham Neubig.
\newblock Openhands: An open platform for ai software developers as generalist agents, 2025{\natexlab{b}}.
\newblock URL \url{https://arxiv.org/abs/2407.16741}.

\bibitem[Wei et~al.(2025)Wei, Sun, Papay, McKinney, Han, Fulford, Chung, Passos, Fedus, and Glaese]{wei2025browsecompsimplechallengingbenchmark}
Jason Wei, Zhiqing Sun, Spencer Papay, Scott McKinney, Jeffrey Han, Isa Fulford, Hyung~Won Chung, Alex~Tachard Passos, William Fedus, and Amelia Glaese.
\newblock Browsecomp: A simple yet challenging benchmark for browsing agents, 2025.
\newblock URL \url{https://arxiv.org/abs/2504.12516}.

\bibitem[Xi et~al.(2025)Xi, Huang, Liao, Huang, Guo, Liu, Zheng, Ye, Zhang, Chen, He, Ding, Li, Chen, Du, Yao, Xu, Chen, Gui, Wu, Zhang, Huang, and Jiang]{xi2025agentgymrl}
Zhiheng Xi, Jixuan Huang, Chenyang Liao, Baodai Huang, Honglin Guo, Jiaqi Liu, Rui Zheng, Junjie Ye, Jiazheng Zhang, Wenxiang Chen, Wei He, Yiwen Ding, Guanyu Li, Zehui Chen, Zhengyin Du, Xuesong Yao, Yufei Xu, Jiecao Chen, Tao Gui, Zuxuan Wu, Qi~Zhang, Xuanjing Huang, and Yu-Gang Jiang.
\newblock Agentgym-rl: Training {LLM} agents for long-horizon decision making through multi-turn reinforcement learning, 2025.
\newblock URL \url{https://arxiv.org/abs/2509.08755}.

\bibitem[Xie et~al.(2024)Xie, Zhang, Chen, Li, Zhao, Cao, Hua, Cheng, Shin, Lei, Liu, Xu, Zhou, Savarese, Xiong, Zhong, and Yu]{xie2024osworld}
Tianbao Xie, Danyang Zhang, Jixuan Chen, Xiaochuan Li, Siheng Zhao, Ruisheng Cao, Toh~Jing Hua, Zhoujun Cheng, Dongchan Shin, Fangyu Lei, Yitao Liu, Yiheng Xu, Shuyan Zhou, Silvio Savarese, Caiming Xiong, Victor Zhong, and Tao Yu.
\newblock Osworld: Benchmarking multimodal agents for open-ended tasks in real computer environments, 2024.
\newblock URL \url{https://arxiv.org/abs/2404.07972}.

\bibitem[Xin et~al.(2024)Xin, Ren, Song, Shao, Zhao, Wang, Liu, Zhang, Lu, Du, Gao, Zhu, Yang, Gou, Wu, Luo, and Ruan]{xin2024deepseekproverv15harnessingproofassistant}
Huajian Xin, Z.~Z. Ren, Junxiao Song, Zhihong Shao, Wanjia Zhao, Haocheng Wang, Bo~Liu, Liyue Zhang, Xuan Lu, Qiushi Du, Wenjun Gao, Qihao Zhu, Dejian Yang, Zhibin Gou, Z.~F. Wu, Fuli Luo, and Chong Ruan.
\newblock Deepseek-prover-v1.5: Harnessing proof assistant feedback for reinforcement learning and monte-carlo tree search, 2024.
\newblock URL \url{https://arxiv.org/abs/2408.08152}.

\bibitem[Xu et~al.(2023)Xu, Hong, Li, Hu, Chen, and Zhang]{xu2023toolbench}
Qiantong Xu, Fenglu Hong, Bo~Li, Changran Hu, Zhengyu Chen, and Jian Zhang.
\newblock On the tool manipulation capability of open-source large language models, 2023.
\newblock URL \url{https://arxiv.org/abs/2305.16504}.

\bibitem[Xue et~al.(2025)Xue, Zheng, Liu, Li, Zheng, Ma, and An]{xue2025simpletir}
Zhenghai Xue, Longtao Zheng, Qian Liu, Yingru Li, Xiaosen Zheng, Zejun Ma, and Bo~An.
\newblock Simpletir: End-to-end reinforcement learning for multi-turn tool-integrated reasoning, 2025.
\newblock URL \url{https://arxiv.org/abs/2509.02479}.

\bibitem[Yang et~al.(2025)Yang, Dou, Guo, Lu, Ju, Deng, and Xin]{yang2025dcpo}
Shihui Yang, Chengfeng Dou, Peidong Guo, Kai Lu, Qiang Ju, Fei Deng, and Rihui Xin.
\newblock Dcpo: Dynamic clipping policy optimization, 2025.
\newblock URL \url{https://arxiv.org/abs/2509.02333}.

\bibitem[Yao et~al.(2022)Yao, Chen, Yang, and Narasimhan]{yao2022webshop}
Shunyu Yao, Howard Chen, John Yang, and Karthik Narasimhan.
\newblock Webshop: Towards scalable real-world web interaction with grounded language agents, 2022.
\newblock URL \url{https://arxiv.org/abs/2207.01206}.

\bibitem[Yao et~al.(2023)Yao, Zhao, Yu, Du, Shafran, Narasimhan, and Cao]{yao2023react}
Shunyu Yao, Jeffrey Zhao, Dian Yu, Nan Du, Izhak Shafran, Karthik Narasimhan, and Yuan Cao.
\newblock React: Synergizing reasoning and acting in language models, 2023.
\newblock URL \url{https://arxiv.org/abs/2210.03629}.

\bibitem[Yu et~al.(2025)Yu, Zhang, Zhu, Yuan, Zuo, Yue, Dai, Fan, Liu, Liu, Liu, Lin, Lin, Ma, Sheng, Tong, Zhang, Zhang, Zhang, Zhu, Zhu, Chen, Chen, Wang, Yu, Song, Wei, Zhou, Liu, Ma, Zhang, Yan, Qiao, Wu, and Wang]{yu2025dapoopensourcellmreinforcement}
Qiying Yu, Zheng Zhang, Ruofei Zhu, Yufeng Yuan, Xiaochen Zuo, Yu~Yue, Weinan Dai, Tiantian Fan, Gaohong Liu, Lingjun Liu, Xin Liu, Haibin Lin, Zhiqi Lin, Bole Ma, Guangming Sheng, Yuxuan Tong, Chi Zhang, Mofan Zhang, Wang Zhang, Hang Zhu, Jinhua Zhu, Jiaze Chen, Jiangjie Chen, Chengyi Wang, Hongli Yu, Yuxuan Song, Xiangpeng Wei, Hao Zhou, Jingjing Liu, Wei-Ying Ma, Ya-Qin Zhang, Lin Yan, Mu~Qiao, Yonghui Wu, and Mingxuan Wang.
\newblock Dapo: An open-source llm reinforcement learning system at scale, 2025.
\newblock URL \url{https://arxiv.org/abs/2503.14476}.

\bibitem[Zhang et~al.(2025)Zhang, Chen, Li, Tu, and Li]{zhang2025rlvmr}
Zijing Zhang, Ziyang Chen, Mingxiao Li, Zhaopeng Tu, and Xiaolong Li.
\newblock Rlvmr: Reinforcement learning with verifiable meta-reasoning rewards for robust long-horizon agents, 2025.
\newblock URL \url{https://arxiv.org/abs/2507.22844}.

\bibitem[Zhao et~al.(2026)Zhao, Xie, Liu, Huang, Pang, Chen, and Grover]{zhao2026selfdistilledreasoneronpolicyselfdistillation}
Siyan Zhao, Zhihui Xie, Mengchen Liu, Jing Huang, Guan Pang, Feiyu Chen, and Aditya Grover.
\newblock Self-distilled reasoner: On-policy self-distillation for large language models, 2026.

\bibitem[Zheng et~al.(2023)Zheng, Chiang, Sheng, Zhuang, Wu, Zhuang, Lin, Li, Li, Xing, Zhang, Gonzalez, and Stoica]{zheng2023judgingllmasajudgemtbenchchatbot}
Lianmin Zheng, Wei-Lin Chiang, Ying Sheng, Siyuan Zhuang, Zhanghao Wu, Yonghao Zhuang, Zi~Lin, Zhuohan Li, Dacheng Li, Eric~P. Xing, Hao Zhang, Joseph~E. Gonzalez, and Ion Stoica.
\newblock Judging llm-as-a-judge with mt-bench and chatbot arena, 2023.
\newblock URL \url{https://arxiv.org/abs/2306.05685}.

\bibitem[Zhou et~al.(2023)Zhou, Liu, Xu, Iyer, Sun, Mao, Ma, Efrat, Yu, Yu, Zhang, Ghosh, Lewis, Zettlemoyer, and Levy]{zhou2023limaalignment}
Chunting Zhou, Pengfei Liu, Puxin Xu, Srini Iyer, Jiao Sun, Yuning Mao, Xuezhe Ma, Avia Efrat, Ping Yu, Lili Yu, Susan Zhang, Gargi Ghosh, Mike Lewis, Luke Zettlemoyer, and Omer Levy.
\newblock Lima: Less is more for alignment, 2023.
\newblock URL \url{https://arxiv.org/abs/2305.11206}.

\bibitem[Zhou et~al.(2024)Zhou, Zanette, Pan, Levine, and Kumar]{zhou2024archer}
Yifei Zhou, Andrea Zanette, Jiayi Pan, Sergey Levine, and Aviral Kumar.
\newblock Archer: Training language model agents via hierarchical multi-turn rl, 2024.
\newblock URL \url{https://arxiv.org/abs/2402.19446}.

\bibitem[Zhou et~al.(2025)Zhou, Jiang, Tian, Weston, Levine, Sukhbaatar, and Li]{zhou2025sweetrl}
Yifei Zhou, Song Jiang, Yuandong Tian, Jason Weston, Sergey Levine, Sainbayar Sukhbaatar, and Xian Li.
\newblock Sweet-rl: Training multi-turn {LLM} agents on collaborative reasoning tasks, 2025.
\newblock URL \url{https://arxiv.org/abs/2503.15478}.

\end{thebibliography}
\bibliographystyle{references}

\newpage
\appendix
\onecolumn

\section{Experiment Details}
\label{app:experiment_details}

This appendix provides domain-specific training procedures, terminology conventions, and detailed ablation analyses for the experiments in Section~\ref{sec:experiments}.

\subsection{Terminology Conventions}
\label{app:terminology}

Throughout the experiments, we call every extracted assistant turn a \emph{pivot candidate} and reserve the term \emph{pivot} for the offline-filtered subset used for PivotRL training. An \emph{action} is the full assistant completion at a model-call boundary, not an individual token. Depending on the benchmark, that completion may be a conversational tool-use turn, a coding-agent tool call, a bash command, or a search/browsing step. Natural-language-only assistant turns are also actions when the benchmark includes them.

In the single-domain experiments, OOD change denotes the score difference relative to the base model on benchmarks outside the training domain. 

For pivot selection, we use two subsets: \emph{random} ($\mathcal{D}_{\mathrm{cand}}$), which uses all pivot candidates, and \emph{low-reward-mean} ($\mathcal{D}_{\mathrm{adv}}$), which requires both nonzero reward variance and a gap between the demonstrated action and the reference-policy mean.

\subsection{Domain-Specific Training Details}
\label{app:training_details}

In all four domains, pivot candidates are extracted once from the demonstration traces, profiled offline, and filtered before RL training; training then samples only from the retained set. We detail the domain-specific action definitions, data construction, and local verifiers below.

\paragraph{$\tau^2$-Bench.}
We curate an SFT trajectory dataset of 281{,}774 trajectories across 838 domains using a synthetic data-generation pipeline similar to ToolAce~\citep{liu2025toolacewinningpointsllm}, Kimi-K2~\citep{kimiteam2025kimik2openagentic}, and DeepSeek-V3.2~\citep{deepseekai2025deepseekv32pushingfrontieropen}. We train on top of \texttt{Qwen3-30B-A3B-Thinking-2507}. Every assistant turn in the trace is first treated as a pivot candidate. In this domain, an action is the full assistant turn at the model-call boundary, which may contain natural language, tool calls, or both. We then evaluate random and low-reward-mean pivot sets as described in Section~\ref{sec:pivotrl_method}. We release the conversational tool use environment and data as part of \href{https://github.com/NVIDIA-NeMo/Gym/tree/main/resources_servers/single_step_tool_use_with_argument_comparison}{Nemo-Gym} and \href{https://huggingface.co/datasets/nvidia/Nemotron-RL-Agentic-Conversational-Tool-Use-Pivot-v1}{Nemotron-Post-Training-v3}.

\paragraph{Terminal-Bench.}
We collect resolved trajectories from \texttt{Qwen/Qwen3-Coder-480B-A35B-Instruct} and \texttt{moonshotai/Kimi-K2-Instruct} using the Terminus-1 and Terminus-2 agents. Every assistant bash action is treated as a pivot candidate, and an action in this domain is the next bash command produced at the model-call boundary. We start from the low-reward-mean set $\mathcal{D}_{\mathrm{adv}}$, with an additional command-deduplication step to improve diversity. The final dataset contains approximately 20{,}000 samples. The local verifier asks whether the sampled command is locally interchangeable with the demonstrated command: it combines output-schema validation, normalized string similarity, and equivalence-based LLM-as-judge scoring over the command and its immediate effect.

\paragraph{SWE-Bench Verified.}
We use an internal trajectory dataset generated with OpenHands~\citep{wang2025openhandsopenplatformai}, OpenCode~\citep{opencode}, and Codex~\citep{openaicodex_software} on tasks from SWE-Gym~\citep{pan2025trainingsoftwareengineeringagents} and R2E-Gym~\citep{jain2025r2egymproceduralenvironmentshybrid}, using \texttt{MiniMaxAI/MiniMax-M2.5}~\citep{minimax2026m25}. We evaluate mean@3 with the OpenHands harness. Every non-error tool-call action is treated as a pivot candidate, and an action in this domain is the next assistant tool call in the coding trace. Filtering to $\mathcal{D}_{\mathrm{adv}}$ yields a final training set of 87{,}718 samples. The local verifier matches tool-call names only. This is a deliberately coarse local signal: it checks whether the model selected the correct next \emph{kind} of operation at the pivot (e.g., search, open, edit, or run), without attempting to score tool arguments or patch quality at every turn. Final task success is determined only by the full SWE-Bench evaluation harness.

For the E2E RL comparison in Section~\ref{sec:swe_e2e}, PivotRL trains with a batch size of $1024$ ($64$ prompts $\times$ $16$ generations) at $1$ turn per sample, reaching $32.67\%$ accuracy at step $130$ with ${\sim}133$K cumulative rollout turns. The E2E RL baseline trains with a batch size of $512$ ($16$ prompts $\times$ $32$ generations) at $12$--$25$ turns per trajectory, reaching the same $32.67\%$ accuracy at step ${\sim}72$ with ${\sim}542$K cumulative rollout turns. At this matched accuracy, PivotRL requires ${\sim}4\times$ fewer rollout turns and ${\sim}5.5\times$ less wall-clock time (Figure~\ref{fig:swe_speedup}). Both methods use the same number of compute nodes.

\paragraph{BrowseComp.}
We start from a multi-hop question-answer dataset and use an online search engine together with \texttt{DeepSeek-V3.2} to generate browsing trajectories. Each search-related assistant step at a model-call boundary is treated as a pivot candidate; in this domain, an action is the next browsing step, such as issuing a search query, opening a result, or taking the next evidence-gathering action. The final dataset contains 13{,}215 samples.

\subsection{Detailed Ablation Analyses}
\label{app:detailed_ablations}

\subsubsection{Effect of Turn Selection Strategy on $\tau^2$-Bench}
\label{app:pivot_selection}

We compare two pivot selection strategies on $\tau^2$-Bench:
\begin{itemize}
    \item random pivots ($\mathcal{D}_{\mathrm{cand}}$): no filtering,
    \item low-reward-mean ($\mathcal{D}_{\mathrm{adv}}$): mixed-outcome pivots further filtered for a large gap between the demonstrated action's reward and the reference-policy mean, given by Eq.~\eqref{eq:dadv}.
\end{itemize}
Table~\ref{tab:tau_pivot_selection} shows a monotonic benefit from more selective filtering: random pivots already improve over SFT ($59.68$ vs.\ $58.44$), and low-reward-mean pivots yield the best result ($63.81$).

\begin{table*}[t]
\small
\centering
\setlength{\tabcolsep}{6pt}
\newcolumntype{F}{>{\centering\arraybackslash}p{0.8cm}}
\newcolumntype{C}{>{\centering\arraybackslash}p{0.8cm}}
\begin{tabular}{l|F|F|C|C}
\toprule
& \multicolumn{1}{c|}{\textbf{$\tau^2$-Airline}}
& \multicolumn{1}{c|}{\textbf{$\tau^2$-Retail}}
& \multicolumn{1}{c|}{\textbf{$\tau^2$-Telecom}}
& \multicolumn{1}{c}{\textbf{$\tau^2$-Average}} \\
\midrule
gpt-oss-120b-high
& 55.33 & 58.19 & 67.20 & 60.24 \\
gpt-oss-20b-high
& 43.33 & 50.20 & 52.63 & 48.72 \\
Qwen3-235B-A22B-Thinking-2507
& 50.00 & 66.08 & 46.78 & 54.29 \\
\midrule
Qwen3-30B-A3B-Thinking-2507
& 50.00 & 54.39 & 28.65 & 44.35 \\
\hdashline
Qwen + SFT
& 51.33 & 58.19 & 65.79 & 58.44 \\
Qwen + Random Pivots ($\mathcal{D}_{\mathrm{cand}}$)
& 58.00 & 63.74 & 57.31 & 59.68 \\
Qwen + Low-Reward-Mean Pivots ($\mathcal{D}_{\mathrm{adv}}$)
& 58.67 & 69.01 & 63.74 & 63.81 \\
\bottomrule
\end{tabular}
\caption{
Effect of pivot selection on $\tau^2$-Bench (best checkpoint over training).
Random turns improve over the base model and low-reward-mean pivots perform even stronger.
}
\label{tab:tau_pivot_selection}
\end{table*}

\section{Omitted Proofs}

\subsection{Proof of Theorem~\ref{thm:pivot}}

\begin{proof}
Fix $s$ and $\beta>0$. Write for simplicity
\begin{align*}
\pi = \pi_{s,\beta}, \quad
q = \mathbb E_{a\sim\pi}[r(s,a)], \quad
\sigma^2 = \mathrm{Var}_{a\sim\pi}(r(s,a)).
\end{align*}

We first compute the natural gradient of $J_s$ at $\pi$.
Let $v\in T_\pi$, and consider the exponential perturbation
\begin{align*}
\pi_t(a)= &~ \frac{\pi(a)e^{t v(a)}}{\sum_{b\in\mathcal{A}(s) }\pi(b)e^{t v(b)}}.
\end{align*}
For any $v\in T_\pi$, we have
\begin{align*}
\left.\frac{d}{dt}\right|_{t=0}\log \pi_t(a)
= v(a)-\sum_{b\in\mathcal{A}(s) }\pi(b)v(b)
= v(a).
\end{align*}
Thus
\begin{align*}
DJ_s(\pi)[v]
= &~ \left.\frac{d}{dt}\right|_{t=0}\mathbb E_{a\sim\pi_t}[r(s,a)] \\
= &~ \left.\frac{d}{dt}\right|_{t=0}\sum_{a\in\mathcal{A}(s) }\pi_t(a)r(s,a) \\
= &~ \sum_{a\in\mathcal{A}(s) }\pi(a)r(s,a)v(a) \\
= &~ \sum_{a\in\mathcal{A}(s) }\pi(a)\bigl(r(s,a)-q\bigr)v(a) \\
= &~ \langle r(s,\cdot)-q,\,v\rangle_{F,\pi}.
\end{align*}
Since this holds for every $v\in T_\pi$, the Riesz representative is
\begin{align*}
\nabla^{\mathrm{nat}}J_s(\pi)= &~ r(s,\cdot)-q.
\end{align*}
We then have
\begin{align}\label{eq:norm-natural-J}
\bigl\|\nabla^{\mathrm{nat}}J_s(\pi)\bigr\|_{F,\pi}^2
= &~ \mathbb E_{a\sim\pi}\bigl[(r(s,a)-q)^2\bigr] \notag\\
= &~ \mathrm{Var}_{a\sim\pi}(r(s,a)) \notag\\
= &~ \sigma^2
\end{align}

Next, we compute the $\beta$-derivative of the exponential-tilted path. Define
\begin{align*}
Z_s(\beta)= &~ \sum_{b\in\mathcal{A}(s) }\pi_0(b\mid s)e^{r(s,b)/\beta}.
\end{align*}
Then
\begin{align*}
\log \pi_{s,\beta}(a)
= &~ \log \pi_0(a\mid s) + \frac{r(s,a)}{\beta} - \log Z_s(\beta).
\end{align*}
Differentiating with respect to $\beta$ gives
\begin{align*}
\partial_\beta \log \pi_{s,\beta}(a)
= &~ -\frac{r(s,a)}{\beta^2} - \frac{Z_s'(\beta)}{Z_s(\beta)}.
\end{align*}
Now
\begin{align*}
Z_s'(\beta)
= &~ \sum_{b\in\mathcal{A}(s) }\pi_0(b\mid s)e^{r(s,b)/\beta}\left(-\frac{r(s,b)}{\beta^2}\right) \\
= &~ -\frac{1}{\beta^2}\sum_{b\in\mathcal{A}(s) }\pi_0(b\mid s)e^{r(s,b)/\beta}r(s,b) \\
= &~ -\frac{Z_s(\beta)}{\beta^2}\,\mathbb E_{b\sim\pi_{s,\beta}}[r(s,b)] \\
= &~ -\frac{q}{\beta^2}Z_s(\beta).
\end{align*}
Therefore
\begin{align*}
\partial_\beta \log \pi_{s,\beta}(a)
= &~ -\frac{1}{\beta^2}\bigl(r(s,a)-q\bigr) \\
= &~ -\frac{1}{\beta^2}\nabla^{\mathrm{nat}}J_s(\pi_{s,\beta})(a).
\end{align*}

If $\sigma=0$, then $r(s,a)=q$ for $\pi$-almost every $a$, so
\begin{align*}
\nabla^{\mathrm{nat}}J_s(\pi)= &~ 0, ~~\text{and}~~ \frac{1}{\beta^2}\bigl\|\nabla^{\mathrm{nat}}J_s(\pi)\bigr\|_{F,\pi}= 0.
\end{align*}

Now assume $\sigma>0$. Then
\begin{align*}
\gamma_{s,\beta}
= &~ -\mathbb E_{a\sim\pi}
\left[
\frac{r(s,a)-q}{\sigma}\,\partial_\beta \log \pi(a)
\right] \\
= &~ -\mathbb E_{a\sim\pi}
\left[
\frac{r(s,a)-q}{\sigma}\left(-\frac{1}{\beta^2}(r(s,a)-q)\right)
\right] \\
= &~ \frac{1}{\beta^2\sigma}\mathbb E_{a\sim\pi}\bigl[(r(s,a)-q)^2\bigr] \\
= &~ \frac{\sigma}{\beta^2}.
\end{align*}
Combining this with Eq.~\eqref{eq:norm-natural-J}, we have
\begin{align*}
\gamma_{s,\beta}
= &~ \frac{1}{\beta^2}\bigl\|\nabla^{\mathrm{nat}}J_s(\pi_{s,\beta})\bigr\|_{F,\pi_{s,\beta}} = \frac{\sigma}{\beta^2}.
\end{align*}
This proves the theorem.
\end{proof}

\subsection{Proof of Theorem~\ref{thm:kl}}\label{sec:proof-kl-proj}

\begin{proof}
Since
\begin{align*}
\mathcal L_{\mathrm{func},\beta}(\pi)= &~ \mathbb E_{s\sim d}[\ell_s(\pi)]
\end{align*}
has no coupling across states, it suffices to solve the pointwise minimization problem for each fixed state $s$. The claim for the population objective then follows immediately $d$-almost surely.

Fix a state $s$, and write $M= \mathcal M(s), ~\rho= \pi_0(M\mid s), ~\pi_0(a)= \pi_0(a\mid s)$.

We first handle the boundary cases. 
If $\rho=0$, then $M=\varnothing$ because $\pi_0(a)>0$ for every $a\in \mathcal{A}(s) $. Hence $r_{\mathrm{func}}(s,a)\equiv 0$, so $\ell_s(\pi)= \beta\, \mathrm{KL}(\pi\|\pi_0)$,
which is uniquely minimized at $\pi=\pi_0$. 
If $\rho=1$, then $M=\mathcal{A}(s) $, so $r_{\mathrm{func}}(s,a)\equiv 1$, and $\ell_s(\pi)= -1 + \beta\, \mathrm{KL}(\pi\|\pi_0)$,
which is again uniquely minimized at $\pi=\pi_0$.

It remains to consider the nondegenerate case $0<\rho<1$.
Let $p$ be any distribution on $\mathcal{A}(s) $, and let
\begin{align*}
q= &~ p(M)=\sum_{a\in M} p(a).
\end{align*}
For $q\in(0,1)$, define the conditional distributions
\begin{align*}
p_+(a)= \frac{p(a)}{q}, \quad a\in M, \qquad
p_-(a)= \frac{p(a)}{1-q}, \quad a\notin M,
\end{align*}
and similarly
\begin{align*}
\pi_{0,+}(a)= \frac{\pi_0(a)}{\rho}, \quad a\in M, \qquad
\pi_{0,-}(a)= \frac{\pi_0(a)}{1-\rho}, \quad a\notin M.
\end{align*}

Using the chain rule for KL divergence, we obtain
\begin{align*}
\mathrm{KL}(p\|\pi_0)
= &~ \sum_{a\in M} p(a)\log\frac{p(a)}{\pi_0(a)}
   + \sum_{a\notin M} p(a)\log\frac{p(a)}{\pi_0(a)} \\
= &~ q\,\mathrm{KL}(p_+\|\pi_{0,+})
   + (1-q)\,\mathrm{KL}(p_-\|\pi_{0,-}) + q\log\frac{q}{\rho}
   + (1-q)\log\frac{1-q}{1-\rho}.
\end{align*}

Since $\mathbb E_{a\sim p}[r_{\mathrm{func}}(s,a)]=q$, the pointwise objective becomes
\begin{align*}
\ell_s(p)
= -q
+ \beta\Bigg[
q\,\mathrm{KL}(p_+\|\pi_{0,+})
+ (1-q)\,\mathrm{KL}(p_-\|\pi_{0,-})
+ q\log\frac{q}{\rho}
+ (1-q)\log\frac{1-q}{1-\rho}
\Bigg].
\end{align*}

For fixed $q$, the last two terms are constants, while the first two KL terms are nonnegative and vanish if and only if
\begin{align*}
p_+= \pi_{0,+}, \quad
p_-= \pi_{0,-}.
\end{align*}
Hence, for fixed $q$, the unique minimizer is
\begin{align*}
\bar p_q(a)= &~
\begin{cases}
\displaystyle \frac{q}{\rho}\,\pi_0(a), & a\in M,\\[1.1ex]
\displaystyle \frac{1-q}{1-\rho}\,\pi_0(a), & a\notin M.
\end{cases}
\end{align*}

Substituting $p=\bar p_q$ eliminates the conditional KL terms and yields the scalar problem
\begin{align*}
\phi_s(q)= &~ -q
+ \beta\left[
q\log\frac{q}{\rho}
+ (1-q)\log\frac{1-q}{1-\rho}
\right],
\qquad q\in[0,1],
\end{align*}
with the usual convention $0\log 0 = 0$.

For $q\in(0,1)$,
\begin{align*}
\phi_s'(q)= &~ -1 + \beta\log\frac{q(1-\rho)}{(1-q)\rho}, \\
\phi_s''(q)= &~ \beta\left(\frac{1}{q}+\frac{1}{1-q}\right) > 0.
\end{align*}
Therefore $\phi_s$ is strictly convex on $(0,1)$ and has at most one stationary point there. Solving $\phi_s'(q)=0$ gives $q= \frac{\rho e^{1/\beta}}{(1-\rho)+\rho e^{1/\beta}}$.
Hence the unique minimizer of $\ell_s$ is $\pi_\beta^\star(\cdot\mid s)=\bar p_{q_\beta(s)}$, namely
\begin{align}\label{eq:pi-beta-star}
\pi_\beta^\star(a\mid s)= &~
\begin{cases}
\displaystyle \frac{q_\beta(s)}{\rho(s)}\,\pi_0(a\mid s), & a\in \mathcal M(s),\\[1.1ex]
\displaystyle \frac{1-q_\beta(s)}{1-\rho(s)}\,\pi_0(a\mid s), & a\notin \mathcal M(s).
\end{cases}
\end{align}
Among all distributions with total acceptable mass $q_\beta(s)$, the unique KL minimizer is precisely the block-rescaled distribution above.

From Eq.~\eqref{eq:pi-beta-star}, the ordering-preservation statements are immediate. If $a,b\in \mathcal M(s)$, then
\begin{align*}
\frac{\pi_\beta^\star(a\mid s)}{\pi_\beta^\star(b\mid s)}
= &~ \frac{\pi_0(a\mid s)}{\pi_0(b\mid s)}.
\end{align*}
Likewise, if $a,b\notin \mathcal M(s)$, then
\begin{align*}
\frac{\pi_\beta^\star(a\mid s)}{\pi_\beta^\star(b\mid s)}
= &~ \frac{\pi_0(a\mid s)}{\pi_0(b\mid s)}.
\end{align*}
This completes the proof.
\end{proof}

\end{document}